\documentclass[letterpaper,11pt]{article}
\usepackage[in]{fullpage} 
\usepackage[utf8]{inputenc}
\usepackage[numbers]{natbib} 
\usepackage{graphicx}
\include{macros.sty}

\title{Private List Learnability vs. Online List Learnability}
\author{Steve Hanneke, Shay Moran, Hilla Schefler, Iska Tsubari}

\begin{document}

\maketitle
\begin{abstract}
	This work explores the connection between differential privacy (DP) and online learning in the context of PAC list learning. In this setting, a $k$-list learner outputs a list of $k$ potential predictions for an instance $x$ and incurs a loss if the true label of $x$ is not included in the list. 
	A basic result in the multiclass PAC framework with a finite number of labels states that private learnability is equivalent to online learnability \cite{AlonLMM19,BunLM20,JungKT20}.
	Perhaps surprisingly, we show that this equivalence does not hold in the context of list learning. 
	Specifically, we prove that, unlike in the multiclass setting, a finite $k$-Littlestone dimension—a variant of the classical Littlestone dimension that characterizes online $k$-list learnability—is not a sufficient condition for DP $k$-list learnability.
	However, similar to the multiclass case, we prove that it remains a necessary condition.
	    
	To demonstrate where the equivalence breaks down, we provide an example showing that the class of monotone functions with $k+1$ labels over $\mathbb{N}$ is online $k$-list learnable, but not DP $k$-list learnable.
	This leads us to introduce a new combinatorial dimension, the \emph{$k$-monotone dimension}, which serves as a generalization of the threshold dimension.
	Unlike the multiclass setting, where the Littlestone and threshold dimensions are finite together, for $k>1$, the $k$-Littlestone and $k$-monotone dimensions do not exhibit this relationship.
	We prove that a finite $k$-monotone dimension is another necessary condition for DP $k$-list learnability, alongside finite $k$-Littlestone dimension.
	Whether the finiteness of both dimensions implies private $k$-list learnability remains an open question.
	
\end{abstract}

\newpage

\section{Introduction}

Modern machine learning applications often involve handling sensitive data, making privacy preservation a critical concern. 
Differential privacy (DP) \cite{DworkMNS06} offers a rigorous framework for safeguarding individuals' information by ensuring that small changes in the input data have a minimal impact on the algorithm’s output. 
In recent years, significant research has addressed the fundamental question of determining which learning tasks can be performed under the constraints of differential privacy. 
A brief overview of relevant works appears in \Cref{subsec:related_work}.

\paragraph{Private PAC learning and online learning.}
An important connection has emerged between private learning and online learning. This connection can be understood through their shared reliance on the concept of stability. Differential privacy is, by definition, a form of stability, as it ensures robustness to small changes in the input data. On the other hand, stability plays a central role in online learning paradigms such as \emph{Follow the Leader} \cite{KalaiV05_online, AbernethyHR08_online, Shalev-ShwartzS07_online, Hazan16_online}.
This relationship has been studied within the framework of \emph{Probably Approximately Correct} (PAC) learning \cite{Valiant84}. It reveals an equivalence: a concept class $\cC \subset \cY^\cX$, where $\cX$ is a domain and $\cY = \{0,1, \ldots, \ell-1\}$ is the label space\footnote{In this work, we focus on the multiclass setting with a finite number of labels, though our results hold more generally
	for an arbitrary number of labels.}, is DP PAC learnable if and only if it is online learnable. This equivalence was first established for the binary case ($\ell=2$) by \cite{AlonLMM19, BunLM20} and later extended to general $\ell$ by \cite{JungKT20,SivakumarBG21}. This work focuses on the relationship between private learnability and online learnability within the context of \emph{list learning}.

\paragraph{List learning.}
 
$k$-List learning is a generalization of supervised classification in which, instead of predicting a single label, the learner outputs a short list of $k$ potential labels, where the goal is for the true label to appear in this list. 
For example, in recommendation systems, it is standard to provide users with a shortlist of items, such as movies, products, or articles, rather than a single suggestion, to improve the likelihood of presenting relevant options. 
Similarly, in scenarios with noisy or ambiguous data, such as medical diagnosis, presenting multiple possible diagnoses is often more practical than requiring a model to pinpoint a single, definitive outcome. Additionally, many real-world decision-making tasks involve downstream human or automated processes, where presenting a shortlist allows for more informed and flexible decision-making.

\medskip

The central question motivating this work is:
\begin{boxF}
	\begin{center}
		\emph{Does the equivalence between private PAC learning and online learning extend to the setting of $k$-list learning for $\ell$-labeled multiclass problems, where $k \geq 2$ and $\ell \geq 3$?} 
	\end{center}
\end{boxF}

\begin{remark}
	The case where $k=1$ and $\ell < \infty$ reduces to the standard multiclass setting. Additionally, the case $k = \ell = 2$ is trivial, as every class is both DP and online $2$-list learnable via a deterministic constant learner that always outputs the list $\{0,1\}$. Thus, the smallest nontrivial case where this question remains unresolved is $k=2$ and $\ell=3$.
\end{remark}

List prediction rules naturally arise in conformal learning. In this setting, a conformal learner predicts possible labels for a test point while also quantifying its confidence. For example, in multiclass classification, a conformal learner assigns scores to each possible class, reflecting the likelihood that the test point belongs to it. The final prediction list is then obtained by selecting the highest-scoring classes. For more details, see \cite{vovk2005algorithmic} and the surveys by \cite{shafer08a_conformal_learning_survey,AngelopoulosB2021_conformal_learning_survey}.

From a theoretical perspective, there has been growing interest in understanding and characterizing the sample complexity of list learning tasks, including PAC $k$-list learning \cite{CharikarP23_list_PAC}, online $k$-list learning \cite{MoranSTY23}, and $k$-list regression \cite{PabbarajuS24_list_regression}. For an overview of relevant works, see \Cref{subsec:related_work}.

\paragraph{Our contribution.}
We provide a negative answer to the motivating question and demonstrate that only one direction of the equivalence between private PAC learnability and online learnability extends to the setting of list learning, making it, to the best of our knowledge, the first known setting where this equivalence fails.
Specifically:
\begin{enumerate}[I]
	\item \emph{Private List Learnability (PLL) $\Longrightarrow$ Online List Learnability (OLL).} We establish that one direction of the equivalence holds in the list learning setting: if a class $\cC$ is DP PAC $k$-list learnable, then it is also online $k$-list learnable. This result leverages the $k$-Littlestone dimension, a variation of the classical Littlestone dimension introduced by \cite{MoranSTY23} to characterize $k$-list online learnability.\footnote{The $k$-Littlestone dimension characterizes $k$-list online learnability in the sense that a class $\cC$ is online $k$-list learnable if and only if its $k$-Littlestone dimension $\LD{\cC}{k}$ is finite \cite{MoranSTY23}.} Conceptually, we demonstrate that deep $k$-Littlestone trees\footnote{A $k$-Littlestone tree is a $(k+1)$-ary shattered mistake tree, as described in \Cref{subsec:preliminaries_learning}.} are an obstacle for private learning: if the $k$-Littlestone dimension of $\cC$ is infinite (and thus $\cC$ is not online $k$-list learnable), then no DP $k$-list learner exists for $\cC$.
	      In other words, finite $k$-Littlestone dimension is a necessary condition for a concept class to be privately $k$-list learnable.
	      \label{itm:main_contribution_1}
	\item \emph{Online List Learnability (OLL) $\centernot\implies$ Private List Learnability (PLL).} 
	      Let $k\geq 2$. We prove that the class of $(k+1)$-labeled monotone functions over $\bbN$ is $k$-list online learnable (with a mistake bound of $1$), but it is \underline{not} DP PAC learnable. More broadly, we introduce a combinatorial dimension, which we term the \emph{$k$-monotone dimension} (\Cref{def:MD}), that serves as a generalization of the threshold dimension. We establish that a finite $k$-monotone dimension is a necessary condition for a concept class to be privately $k$-list learnable.\label{itm:main_contribution_2}
\end{enumerate}

Note that, in the classical multiclass setting, for any class $\cC$, the threshold dimension $\TD{\cC}$ and the Littlestone dimension $\LD{\cC}{}$ are simultaneously finite (that is, if one is finite, so is the other \cite{Shelah1982,HodgesModelTheory}). Consequently, the finiteness of the threshold dimension fully captures the online learnability of a class. However, in the context of list learning, this relationship no longer holds: there exist classes with a finite $k$-Littlestone dimension but an infinite $k$-monotone dimension, and vice versa (\Cref{thm:LD_vs_MD}).
To summarize, we have established that the finiteness of \ul{both} the $k$-Littlestone dimension and the $k$-monotone dimension is a necessary condition for private $k$-list learning (Corollary \ref{maincor:LD_and_MD_are_necessary}).
An open question remains whether the finiteness of these two dimensions is also a sufficient condition for private $k$-list learning.

Finally, another contribution we highlight is of a more technical nature and concerns the use of Ramsey-theoretic tools in our proofs. 

In establishing our two main results, we employ distinct variants of Ramsey's theorem. For the result in \Cref{itm:main_contribution_2} (OLL $\nRightarrow$ PLL), we use the classical Ramsey theorem for hypergraphs, whereas for the result in \Cref{itm:main_contribution_1} (PLL $\Rightarrow$ OLL), we extend a Ramsey theorem for binary trees introduced by \cite{FioravantiHMST24ramsey}. Specifically, we require a version of this theorem for trees with arity \( k+1 \), where \( k \) is the list size.  

\cite{FioravantiHMST24ramsey} developed their Ramsey theorem for trees to overcome a limitation of the classical Ramsey theorem for hypergraphs in proving that private learning implies online learning for general classification problems. At first glance, one might expect the tree-based Ramsey theorem to be strictly stronger in this context. However, we show that these two variants are fundamentally incomparable: in both directions, there exist cases where one applies while the other does not. For a more detailed discussion, see \Cref{subsec:main_results}.

\paragraph{Organization.}
In the remainder of this section, we present our main results and discuss additional related work. We outline the proof of the $k$-Littlestone dimension lower bound in \Cref{sec:PPL_implies_LD} and the proof of the $k$-monotone dimension lower bound in \Cref{sec:PPL_implies_MD}. Detailed definitions and preliminaries are provided in \Cref{sec:preliminaries}. The complete proofs of our main results appear in \Cref{sec:LDvsMD,sec:proof_of_thm_A,sec:proof_of_thm_B}, with additional proofs deferred to \Cref{sec:add_proofs,sec:ramsey_theorem__for_trees}.

\subsection{Main Results}\label{subsec:main_results}

This section presents the main contributions of the paper.
We use standard definitions and terminology from learning theory and differential privacy; see \Cref{sec:preliminaries} for detailed definitions.
In this work, we focus on the case where the label space $\cY$ is finite, though the results extend to infinite label spaces.

\medskip

The $k$-Littlestone dimension is a combinatorial dimension introduced by \cite{MoranSTY23} that generalizes the classical Littlestone dimension and characterizes optimal mistake and regret bounds in online $k$-list learning. 
Unlike the Littlestone dimension, which is defined using binary mistake trees, the $k$-Littlestone dimension is based on $(k+1)$-ary mistake trees, meaning each internal vertex has outdegree $k+1$, where $k$ corresponds to the list size of the learner.
For the formal definition, refer to \Cref{subsec:preliminaries_learning}.
The following theorem provides a lower bound on the sample complexity of privately $k$-list learning a concept class~$\cC$ in terms of its $k$-Littlestone dimension.
Therefore, a finite $k$-Littlestone dimension is a necessary condition for DP PAC $k$-list learning.

\begin{theorem}\label{thm:DP_implies_LD}
	Let $k\geq1$. Let $\cC\subset\cY^\cX$ be a concept class with $k$-Littlestone dimension ${\LD{\cC}{k}\geq d}$, 
	and let $\cA$ be an \textcolor{black}{$\left(\frac{k!
		}{10^4(k+1)^{k+2}},\frac{k!
		}{10^4(k+1)^{k+2}}\right)$}-accurate $k$-list learning algorithm for $\cC$ with sample complexity $m$, satisfying $(\epsilon,\delta(m))$-differential privacy for $\epsilon=\textcolor{black}{\log\left(\frac{400k^2+1}{400k^2}\right)}$ and ${\delta(m)\leq \textcolor{black}{\frac{1}{200k^2m^2}}}$.
	Then, the following bound holds:
	\begin{equation*}
		m = \Omega(\log^\star d), 
	\end{equation*}
	where the $\Omega$ notation conceals a universal numerical multiplicative constant.
\end{theorem}

The proof outline of \Cref{thm:DP_implies_LD} and its main ideas are presented in \Cref{sec:PPL_implies_LD}, while the full proof appears in \Cref{sec:proof_of_thm_A}.

\medskip

In the multiclass setting, the proof of \cite{AlonLMM19,JungKT20} showing that private learnability implies online learnability utilizes the tight connection between the Littlestone and threshold dimensions—one is finite if and only if the other is finite. In their proof, the authors first establish a lower bound on the sample complexity of privately learning threshold functions.
The threshold dimension of a class $\cC \subset \cY^\cX$ is the largest number $d$ such that there are $d$ thresholds embedded in~$\cC$. That is, there exist points $x_1, \ldots, x_d \in \cX$, functions $c_1, \ldots, c_d \in \cC$, and labels~$y_1 < y_2 \in \cY$ such that for all $i, j$, we have $c_i(x_j) = y_1$ if $i \leq j$, and $c_i(x_j) = y_2$ if $i > j$.

We now define the $k$-monotone dimension, a generalization of the threshold dimension.

\paragraph{Monotone functions.} Let $\cX$ be a linearly ordered domain and let $\cY$ be a linearly ordered label space. A function $c:\cX\to \cY$ is a $\cY$-labeled \emph{monotone function} over $\cX$ if $x_1<x_2$ implies $c(x_1)\leq c(x_2)$. 
Alternatively, $c$ is monotone if there exist labels
$i_0<i_1<\ldots <i_m\in\cY$ 
and points $x_1<\ldots<x_m\in\cX$,
such that:

\begin{equation*}
	c(x)=
	\begin{cases}
		i_0 & \text{if } x<x_{1}                                            \\
		i_j & \text{if } x\in [x_{j},x_{j+1}) \text{ for $j=1,\ldots ,m-1$} \\
		i_m & \text{if } x\geq x_m                                          
	\end{cases}.
\end{equation*}
Note that threshold functions are a special case of monotone functions.

\begin{definition}[$k$-Monotone Dimension]\label{def:MD}
	Let $\cC\subset\cY^\cX$ be a class. The \emph{$k$-monotone dimension} of $\cC$, denoted $\MD{\cC}{k}$, is the largest number $d$ such that the following holds:
	There exist 
	\begin{itemize}
		\item [(i)] points $x_1,\ldots,x_d\in\cX$,
		\item  [(ii)] a subset of labels $K\subset\cY$ of size $|K|=k+1$,
		\item [(iii)] and a linear ordering of $K$,
	\end{itemize}
	such that the restriction 
	\[\restr{\cC}{\{x_1,\ldots,x_d\}}=\left\{ c:\{x_1,\ldots,x_d\}\to \cY ~\middle|~ c\in\cC\right\}\] contains all the $K$-labeled monotone functions over $\{x_1<\ldots<x_d\}$.
	If such numbers $d$ can be arbitrarily large, we say $\MD{\cC}{k}=\infty$.
	
\end{definition}
Note that when $k=1$, the $1$-monotone dimension is exactly equal to the threshold dimension.
\medskip

The following theorem provides a lower bound on the sample complexity of privately $k$-list learning a concept class~$\cC$ in terms of its $k$-monotone dimension, establishing that finite $k$-monotone dimension is necessary for DP PAC $k$-list learning.

\begin{theorem}\label{thm:sc_of_monotone_functions}
	Let $k\geq1$. Let $\cC\subset\cY^\cX$ be a concept class with $k$-monotone dimension ${\MD{\cC}{k}\geq d}$, 
	and let $\cA$ be an $\left(\frac{1}{200k(k+1)},\frac{1}{200k(k+1)}\right)$-accurate $k$-list learning algorithm for $\cC$ with sample complexity $m$, satisfying $(\epsilon,\delta(m))$-differential privacy for $\epsilon=0.1$ and $\delta(m)\leq \frac{1}{6(200km)^4\log^2(200km)}$.
	Then, the following bound holds:
	\begin{equation*}
		m = \Omega(\log^\star d), 
	\end{equation*}
	where the $\Omega$ notation conceals a universal numerical multiplicative constant.
	
\end{theorem}
The proof idea of \Cref{thm:sc_of_monotone_functions} is outlined in \Cref{sec:PPL_implies_MD}, with the full proof in \Cref{sec:proof_of_thm_B}. 

\medskip

In contrast to the multiclass setting, the $k$-Littlestone and $k$-monotone dimensions do not exhibit the same relationship. Specifically, the next theorem demonstrates that the gap between them can be infinite.

\begin{theorem}[$k$-LD vs. $k$-MD]\label{thm:LD_vs_MD}
	Let $k>1$. There exist classes $\cC_L,\cC_M$ such that
	\begin{itemize}
		\item [(i)] $\LD{\cC_L}{k}=1$ and $\MD{\cC_L}{k}=\infty$, 
		\item [(ii)] $\LD{\cC_M}{k}=\infty$ and $\MD{\cC_M}{k}=1$.
	\end{itemize}   
\end{theorem}

The proof of \Cref{thm:LD_vs_MD} appears in \Cref{sec:LDvsMD}.

By combining \Cref{thm:DP_implies_LD,thm:sc_of_monotone_functions}, we derive a more comprehensive result:

\begin{boxH}
	\begin{maincorollary}[$k$-LD and $k$-MD are Necessary for PLL]\label{maincor:LD_and_MD_are_necessary}
		Let $\cX$ be a domain and $\cY$ a label space.
		If a concept class $\cC\subset \cY^\cX$ is DP PAC  $k$-list learnable for $k\geq 1$, then
		\begin{itemize}
			\item [(i)] $\LD{\cC}{k}<\infty$, and
			\item [(ii)] $\MD{\cC}{k}<\infty$.
		\end{itemize}
	\end{maincorollary}
	
\end{boxH}

Our lower bounds for private list learnability follow two different approaches introduced by \cite{AlonLMM19} and \cite{FioravantiHMST24ramsey} for proving lower bounds on private learning.
\cite{AlonLMM19} established that private learnability implies online learnability by proving a lower bound on privately learning thresholds, leveraging the classical Ramsey theorem for hypergraphs. However,
since the tight connection between thresholds and the Littlestone dimension does not extend to more general settings—such as partial concept classes \cite{LongPartial,AlonHHM21} and multiclass learning with infinitely many labels—this approach cannot be applied directly.
Later, \cite{FioravantiHMST24ramsey} provided an alternative proof by reasoning directly about Littlestone trees, applying a Ramsey theorem for trees without relying on thresholds as an intermediate step. This approach resolved the open problem for partial concept classes and multiclass settings with infinitely many labels.
While the proof of \cite{FioravantiHMST24ramsey} may appear more general—since Ramsey theorem for trees was specifically developed to analyze Littlestone trees, which characterize online learnability, and has led to lower bounds in a broader range of settings—our results suggest that the two techniques are ultimately incomparable.

In the multiclass case, thresholds and Littlestone trees serve as essentially the same barrier for private learning. However, in the context of list learning, monotone functions and $k$-Littlestone trees form distinct and incomparable barriers. Both finite $k$-Littlestone and $k$-monotone dimensions are necessary for private $k$-list learning, with examples where one is finite while the other infinite.
It remains an open question whether the finiteness of both $k$-Littlestone and $k$-monotone dimensions is sufficient for private $k$-list learning.

\begin{boxH}
	\begin{question*}\label{open question}
		Let $k>1$ and let $\cC\subset\cY^\cX$ be a concept class with $\LD{\cC}{k},\MD{\cC}{k}<\infty$. Is~$\cC$ DP PAC $k$-list learnable?
	\end{question*}
\end{boxH}

While no single combinatorial parameter is currently known to characterize private list learnability, the fact that both the $k$-Littlestone dimension and the $k$-monotone dimension are necessary—but neither is sufficient—suggests that a more nuanced, potentially multi-parameter or non-combinatorial characterization may be required. See \Cref{subsec:related_work} for further discussion.

Finally, the next theorem summarizes the relationship between private list learnability and online list learnability.

\begin{boxH}
	
	\begin{maincorollary}[PLL $\substack{\Rightarrow\\\nLeftarrow}$ OLL]\label{maincor:PPLvsOLL}
		In the items below, $\cX$ denotes an arbitrary domain and $\cY$ denotes a finite label space.
		\begin{itemize}
			\item [(i)] For every $k\geq 1$, if $\cC\subset \cY^\cX$ is DP PAC $k$-list learnable then it is online $k$-list learnable. 
			      Moreover, this implication applies to arbitrary (possibly infinite) label space $\cY$.
			\item[(ii)] For $k=1$, if $\cC\subset \cY^\cX$ is online $k$-list learnable then it is DP PAC $k$-list learnable.
			\item [(iii)] For every $k>1$, there exists a concept class $\cC_k\subset\{0,\ldots,k\}^\bbN$ that is online $k$-list learnable but is not DP PAC $k$-list learnable.
		\end{itemize}
	\end{maincorollary}
\end{boxH}

Corollary \ref{maincor:PPLvsOLL} follows directly from \Cref{thm:LD_vs_MD} and Corollary \ref{maincor:LD_and_MD_are_necessary}, with the latter being a consequence of \Cref{thm:DP_implies_LD,thm:sc_of_monotone_functions}. 
In particular,
Item \emph{(i)} in Corollary \ref{maincor:PPLvsOLL} follows directly by \Cref{thm:DP_implies_LD}.
On the other hand, Item \emph{(iii)} follows as a corollary of \Cref{thm:LD_vs_MD,thm:sc_of_monotone_functions}. Specifically, the class $\cC_L$ from \Cref{thm:LD_vs_MD} is online $k$-list learnable because it has finite $k$-Littlestone dimension. However, since its $k$-monotone dimension is unbounded \Cref{thm:sc_of_monotone_functions} implies it is not DP PAC $k$-list learnable.
The case $k=1$ corresponds to the known result in the multiclass setting \cite{AlonLMM19,BunLM20,JungKT20}.

\subsection{Related Work}\label{subsec:related_work}

\paragraph{Private learning.}

The study of the PAC learning model under differential privacy was initiated by \cite{KasiLNRS11}, who showed that every finite concept class $\cC$ is privately learnable with a sample complexity of $O(\log|\cC|)$. However, this bound is loose for many specific concept classes of interest and offers no guarantees for infinite classes.
One of the most fundamental and well-studied classes is the class of linear classifiers (also known as threshold functions) in $\bbR^d$. While this class is well known to be PAC learnable, a seminal result by \cite{BunNSV15} demonstrates that it is impossible to properly DP learn this class, establishing a lower bound on the sample complexity for one-dimensional thresholds. Subsequently, in his thesis, \cite{BunThesis} provided an alternative proof of this result using a Ramsey-theoretic argument.
Later, \cite{AlonLMM19} extended this result in the context of binary classification beyond proper learning and showed that if a class $\cC$ has Littlestone dimension $d$ (and hence threshold dimension of at least $\log d$), then every (possibly improper) private PAC learner for $\cC$ requires at least $\Omega(\log^\star d)$ samples, once again leveraging Ramsey's theorem. \cite{JungKT20} and \cite{SivakumarBG21} extended this result to the multiclass setting with a finite number of labels, while \cite{FioravantiHMST24ramsey} further generalized it to the multiclass setting with an arbitrary number of labels and to partial concept classes, by developing Ramsey-type theorems for binary trees.
On the other hand, regarding the opposite direction, \cite{BunLM20} showed that every Littlestone class is privately learnable with a sample complexity of $2^{O(2^d)}$, where $d$ denotes the Littlestone dimension. This bound was later improved to $\tilde{O}(d^6)$ by \cite{GhaziG0M21}.

Another line of research examines private learning through the lens of stability, uncovering interesting connections to replicability and reproducibility, low information complexity, PAC-Bayes stability, and other variants of algorithmic stability \cite{ImpagliazzoLPS22_reproducability,BunGHILPSS23,LivniM20,MoranSS23_zoo,PradeepNG22,MM22_unstable_formula}.

Furthermore, extensive research has been dedicated to understanding which learning tasks can be performed under \emph{pure} differential privacy. \cite{beimel2013characterizing, Beimel19Pure} introduced the concept of \emph{representation dimension}, a quantity that characterizes pure DP learnability. In a subsequent work, \cite{FeldmanX15} discovered an interesting connection with communication complexity, associating every concept class $\cC$ with a communication task whose complexity determines whether $\cC$ is pure DP learnable. Additionally, \cite{alon2023unified} provided a unified characterization for both pure and approximate differential privacy, using cliques and fractional cliques of a graph corresponding to $\cC$.

\paragraph{List learning.}

In learning theory, the framework of list learning was introduced by \cite{brukhim2022characterization} as a means to characterize multiclass PAC learnability. Since then, it gained traction as an area of study in its own right.
Notably, \cite{CharikarP23_list_PAC} provided a formal characterization of list PAC learnability, while \cite{MoranSTY23} independently characterized list learnability within the online model. \cite{HannekeMW24_lis_compression} examined classical principles related to generalization, such as uniform convergence, Empirical Risk Minimization, and sample compression, within the context of list PAC learning. More recently, \cite{PabbarajuS24_list_regression} extended the scope of list learning to regression problems. Furthermore, list learning has played a significant role in the study of boosting, with notable contributions from \cite{BrukhimHM23,BrukhimDMM23,bressan2024dice_games_theory_generalized}. It also has strong connections to the \emph{Precision and Recall} learning framework, introduced by \cite{CohenMMS24_precision_recall}.

\paragraph{Combinatorial parameters in learning theory.}
Learning-theoretic phenomena occasionally require combinations of distinct combinatorial parameters rather than single-dimensional characterizations. One notable example is proper learning.
In both PAC and online settings, proper learnability depends simultaneously on the VC dimension and dual Helly number \cite{BousquetHMZ20,KaneLMY19,HannekeLM21}. Interestingly, the dual Helly number may remain finite even when the VC dimension is unbounded.
Another example that exhibits a related phenomenon is list replicability. While a class has finite list replicability number if and only if it has finite Littlestone dimension, no quantitative relationship exists between them—some classes exhibit arbitrarily large Littlestone dimension yet has list replicability number as low as 2. This disconnect motivates the conjecture that the list replicability number may instead be governed by VC dimension \cite{0001CMY24,0001MY23}.
These examples demonstrate that, while uncommon, certain learning-theoretic phenomena may inherently require multiple combinatorial parameters. This raises the possibility that private list learnability, too, may not be captured by a single parameter, but rather by a more intricate combination of factors.

\section{Technical Highlights}\label{sec:technival_overview}

\subsection{Private $k$-List-Learnability Implies Finite $k$-Littlestone Dimension}\label{sec:PPL_implies_LD}

In this section, we outline the proof of \Cref{thm:sc_of_monotone_functions} and establish a lower bound on the sample complexity of privately learning $k$-Littlestone classes. The full proof is presented in \Cref{sec:proof_of_thm_A}.
Our proof follows the approach of \cite{FioravantiHMST24ramsey} for deriving a lower bound on the sample complexity of privately learning ($k=1$)-Littlestone classes, with several necessary adaptations. We begin with a general overview of the proof and, in \Cref{subsec:comparison_with_ramsey_for_trees}, provide a detailed comparison between our proof and that of \cite{FioravantiHMST24ramsey}.

The first step of the proof is to show the following. Given a $(k+1)$-ary mistake tree $T$ and any $k$-list algorithm $\cA$ that takes $T$-realizable\footnote{A $T$-realizable sample is a sample that is realizable by a branch of $T$.} samples as input, there exists a large subtree of $T$ where the loss of $\cA$ on every test point $x$ depends only on comparisons within the tree. (We elaborate on this below.)
The second step of the proof hinges on a reduction from the \emph{interior point problem}, which was introduced by \cite{BunNSV15} in the context of properly learning thresholds (see \Cref{subsec:preliminaries_DP} for the formulation of the problem).

\paragraph{Step 1: Reduction to algorithms with comparison-based loss.}
The starting point of the proof is to show that, given a $(k+1)$-ary mistake tree $T$, for every algorithm $\cA$ we can find a subtree of~$T$ on which 
the behavior of $\cA$
is based on comparisons with respect to the natural partial order on $T$.
A result of comparing two examples $x'$ and $x''$ can be one of the following outcomes: either
\begin{enumerate}
	\itemsep0em
	\item $x'$ is a $i$-th descendant\footnote{A vertex $v$ is the $i$-th descendant of $u$ in a $b$-ary tree $T$ if $v$ belongs to the subtree rooted at the child of $u$ corresponding to its $i$-th outgoing edge.} of $x''$, where $i=0,\ldots,k$,
	\item $x''$ is a $i$-th descendant of $x'$, where $i=0,\ldots,k$, or
	\item  $x'$ and $x''$ are incomparable.
\end{enumerate}
Notice that there are a total of~$2(k+1)+1$ possible comparison outcomes.
Now, what does it mean for an algorithm to behave in a comparison-based manner? 
\cite{FioravantiHMST24ramsey} considered algorithms whose \textbf{predictions} are based solely on comparisons. Specifically, for every $T$-realizable input sample $S$ and any test point~$x$, the prediction $\cA(S)(x)$ depends only on the comparisons of the test point $x$ with points in~$S$. 
This is a strong algorithmic guarantee, which is key to derive the impossibility result. However, it is also a natural notion, as many reasonable algorithms exhibit comparison-based behavior.

Attempting to extend the definition of comparison-based predictions to the setting of $k$-list algorithms would introduce unnecessary complications to the proof. Additionally, it would lead to a deterioration of the bounds, making them dependent on the arity of the tree. This outcome is counterintuitive because, as the size of the list available to the algorithm increases, one would naturally expect it to have greater predictive power.
To circumvent this issue, we 
revisit the proof of \cite{FioravantiHMST24ramsey} and 
identify the key property necessary for the second step of the proof (the reduction from the interior point problem). Specifically, it is sufficient for the algorithm to have comparison-based \textbf{loss}, rather than requiring its entire predictions to be comparison-based.
Roughly speaking, a $k$-list algorithm $\cA$ has \emph{comparison-based loss} with respect to $T$ if the following holds. For every $T$-realizable input sample $S$ and any test point~$x$ on the branch realizing~$S$, the probability that the list $\cA(S)(x)$ does not contain the correct label of~$x$ depends only on the comparisons of the test point $x$ with points in~$S$.
We remark that, in the case of binary
Littlestone trees (i.e., $k=1$), the definitions of comparison-based loss and comparison-based predictions coincide.

Finally, concluding this step required generalizing the Ramsey theorem for trees. While \cite{FioravantiHMST24ramsey} proved this result for binary trees to show that any algorithm has comparison-based predictions on a deep subtree, we extended it to $b$-ary trees. This generalization allowed us to demonstrate that, for every $k$-list algorithm, there exists a deep subtree on which it exhibits comparison-based loss.

\paragraph{Step 2: Reduction from the interior point problem.}
The second step of the proof consists of establishing a lower bound on the sample complexity of private $k$-list algorithms with comparison-based loss. We do so by showing a reduction from the interior point problem. A randomized algorithm solves the interior point problem on~$[n]$ if for every input dataset $X\in[n]^m$, with high probability it returns a point that lies between $\min X$ and $\max X$. \cite{BunNSV15} showed that solving the interior point problem in a private manner requires a dataset size of $m\geq\Omega(\log^\star n)$ (see \Cref{thm:lower_bound_ipp}). We use this result to derive a lower bound in our setting.

Let $T$ be a $(k+1)$-mistake tree of depth $n$ and let $\cA$ be a private empirical $k$-list learner\footnote{A $k$-list learner is empirical learner for $T$ if it is an empirical learner with respect to input samples that are realizable by (a branch of) $T$. It is enough to consider private empirical PAC learners, since any private PAC learner can be transformed into a private empirical learner, while the sample complexity is increased only by a multiplicative constant factor. See \Cref{lemma:reduction_to_empirical_learner}.}
for~$T$ with comparison-based loss. Let $d_1<\ldots<d_m\in [n]$ be the input for the interior point problem. 
The reduction proceeds as follows.
\begin{enumerate}
	\itemsep0em 
	\item Pick a branch $B$ in $T$ uniformly at random, and associate each point $d_i$ with the point~$x_i$ on~$B$ at depth~$d_i$. This defines an input sequence $S=\big((x_1,y_1),\ldots,(x_m,y_m)\big)$, where the labels $y_i$’s are determined by the branch $B$.
	\item Run $\cA$ on~$S$ and search for an interval of length $l$ on~$B$, where $l$ is sufficiently large, such the accumulated loss over this interval is $\leq \frac{1}{2(k+1)}\cdot l$.\label{itm:resuction_item_2}
	      Note that this loss guarantee outperforms guessing a random list of size $k$, which would result in an expected accumulated loss of $\frac{1}{k+1}\cdot l$.
	\item Return the depth of the first point in the deepest such interval.
\end{enumerate}
The idea that stands behind the reduction is that the predictions of the $k$-list learner have high correlation with the branch on which the input sample $S$ lies on.
Specifically, we show that with high probability, the output of the reduction is an interior point of $\{d_1, \ldots, d_m\}$. This is achieved by proving two key properties: 
\begin{itemize}\label{itm:steps_to_prove_reduction_item_a}
	\itemsep0em 
	\item [(a)] With high probability, there exists such a correlated interval (as described in \Cref{itm:resuction_item_2}) beginning between $x_1$ and $x_m$. 
	\item [(b)]It is very unlikely that such an interval begins after $x_m$.\label{item:}
\end{itemize}
The second item is simple because $\cA$ can only access information up to depth $d_i$, and the part of the (random) branch below it is therefore independent of $\cA$'s output hypothesis. Hence, below this depth, $\cA$ cannot significantly outperform random guessing.

The first item is more challenging, and its proof heavily relies on the fact that $\cA$ is differentially private and has a comparison-based loss. The challenge in deriving the first item arises from the fact that an empirical learner is only guaranteed to be accurate on its training set, whereas we seek accuracy over a long continuous interval that includes many examples outside the training set. We address this by leveraging the comparison-based nature of the algorithm to argue that its loss remains the same on similar points lying between consecutive training examples, and by utilizing its differential privacy to shift training examples as needed.

We note that this part of the proof is also the densest one in \cite{FioravantiHMST24ramsey}’s argument for the case of $k = 1$. Whether their analysis of this part extends to $k$-list learning for arbitrary $k$ is unclear; at the very least, a naive extension would significantly complicate the calculations and lead to an unreasonably large case analysis.
We circumvent this complication by identifying a way to simplify \cite{FioravantiHMST24ramsey}’s argument. Not only does this simplification make the extension to general~$k$ more attainable, but it also simplifies the original proof by \cite{FioravantiHMST24ramsey}, both conceptually and technically.
We discuss this step further in the next section.

\subsubsection{Comparison to \cite{FioravantiHMST24ramsey}}\label{subsec:comparison_with_ramsey_for_trees}

\paragraph{Comparison-based loss vs. comparison-based predictions.}
As elaborated above, one difference between our proof and that of \cite{FioravantiHMST24ramsey} is that they considered algorithms whose entire predictions, rather than just their loss, are comparison-based. While \cite{FioravantiHMST24ramsey}'s notion of comparison-based predictions is more intuitive than our notion of comparison-based loss, using the latter significantly simplifies the proof from a technical perspective and yields better bounds. 

\medskip

We now discuss a more substantial difference, which not only enabled the extension of \cite{FioravantiHMST24ramsey}’s proof approach to list learning but also provided an insight that simplifies their original proof.

\paragraph{Analysis of the reduction from interior point problem.}
Recall that the main challenge in analyzing the reduction is to leverage the properties of the learner—specifically, privacy and comparison-based loss—to extend the guarantee on $\cA$'s empirical loss from the sample to an entire interval (Item~\hyperref[itm:steps_to_prove_reduction_item_a]{(a)} above).

In their proof, \cite{FioravantiHMST24ramsey} showed that, with high probability, one can find two consecutive points $x_i$ and $x_{i+1}$ on the sample $S$ such that:
\begin{enumerate}
	\itemsep0em 
	\item The empirical loss of $\cA(S)$ on $x_i$ and $x_{i+1}$ is relatively small, \label{itm:reduction_ipp_item1}
	\item The labels of $x_i$ and $x_{i+1}$ are different, and\label{itm:reduction_ipp_item2}
	\item There exist two ``matching neighbors" on the branch~$B$. That is, there are points~$x'$ and~$x''$ such that~$x'$ lies on the branch $B$ between $x_{i-1}$ and $x_i$, and~$x''$ lies on~$B$ between~$x_{i+1}$ and~$x_{i+2}$. Furthermore, the label of $x'$ matches the label of $x_i$, and the label of $x''$ matches the label of~$x_{i+1}$.\label{itm:reduction_ipp_item3}
\end{enumerate}
Then, they argued that the loss of $\cA(S)$ on the interval between $x_i$ and $x_{i+1}$ is small as follows. Let~$x$ be a point on the interval between $x_i$ and $x_{i+1}$. Since the tree is binary, the label of $x$ matches the label of one of $x_i$ or $x_{i+1}$ (by \Cref{itm:reduction_ipp_item2}). Suppose it matches the label of $x_i$. Now, define a new sample $S'$ obtained by replacing $x_i$ with its matching neighbor $x'$. Then, the following holds:
\[
	\cA(S)(x) \overset{\text{DP}}{\approx} \cA(S')(x) \overset{\text{CB}}{\approx} \cA(S')(x_i) \overset{\text{DP}}{\approx} \cA(S)(x_i),
\]
where ``$\overset{\text{DP}}{\approx}$'' denotes approximate equality due to differential privacy, and ``$\overset{\text{CB}}{\approx}$'' denotes approximate equality due to comparison-based behavior. 
Therefore, the empirical loss of $\cA$ on an internal point in the interval is controlled by the loss at the endpoints, which is small (by \Cref{itm:reduction_ipp_item1}).

When extending this analysis to our setting of \( k \)-list learning, we encountered several challenges. The first challenge arises from the fact that if we only consider two points in the sample with different labels, we can only guarantee accuracy on points with those specific two labels, rather than across the entire interval. This is problematic because minimizing loss on just two labels can always be trivially achieved by including both labels in the prediction list.  

A natural way to address this issue is to consider \( k \) consecutive intervals whose endpoints have~\( k+1 \) distinct labels and to argue that the algorithm maintains small loss over the \underline{union} of these \( k \) intervals. However, reasoning over a union of \( k \) intervals introduces a fundamental obstacle: finding appropriate matching neighbors becomes effectively impossible. Specifically, shifting training-set points without altering their labels—an essential step in \cite{FioravantiHMST24ramsey}'s approach—no longer seems feasible, making it unclear how to ensure the algorithm's loss remains small on every point in the union of the \( k \) intervals.  

Can this obstacle be overcome? A closer inspection of \cite{FioravantiHMST24ramsey}'s proof reveals that they took great care to ensure that the comparisons between a point \( x \) and the original sample~\(S\), as well as between a point \( x_i \) and the modified sample \( S' \), remained identical. This corresponds to the matching neighbors requirement (\Cref{itm:reduction_ipp_item3}). Our key insight was that by leveraging the full strength of the Ramsey theorem for trees—specifically, a version that ensures all chains have colors determined solely by their order type—we gain significantly more flexibility in shifting training examples. This removes the stringent constraints that arise when attempting to extend \cite{FioravantiHMST24ramsey}'s original analysis, allowing us to circumvent the fundamental obstacle discussed above.
This insight also leads to a simplification of \cite{FioravantiHMST24ramsey}'s proof by eliminating the matching neighbors constraint in \Cref{itm:reduction_ipp_item3}.

\subsection{Private $k$-List-Learnability Implies Finite $k$-Monotone Dimension}\label{sec:PPL_implies_MD}
In this section, we outline the proof of \Cref{thm:sc_of_monotone_functions} that establishes a lower bound on the sample complexity of privately learning monotone functions. The full proof appears in \Cref{sec:proof_of_thm_B}. Our proof follows the approach of \cite{AlonLMM19} for deriving a lower bound on the sample complexity of privately learning thresholds, with several necessary adaptations. It consists of two main steps.  

First, we show that for any $k$-list algorithm, we can identify a large subset~$\cX'\subseteq \cX$ where the algorithm's predictions depend only on comparisons.  
In the second step, we use a packing argument to show that for private $k$-list algorithms that learn monotone functions, $\cX'$ cannot be too large. Combining the bounds derived in each step yields the desired lower bound on the sample complexity.

To provide context, we briefly describe the proof of \cite{AlonLMM19}. Their approach exploits the structure of one-dimensional thresholds and applies the classical Ramsey theorem to identify a \emph{homogeneous set} of large size.  A subset~$\cX'$ of an ordered domain~$\cX$ is considered homogeneous with respect to an algorithm~$\cA$ if whenever the input sample $S$ and the test point $x$ are from $\cX'$, then the prediction of~$\cA$ for a test point~$x$ depends solely on the relative position of~$x$ within the sorted input sample~$S$. In other words, the algorithm behaves in a comparison-based manner on inputs from $\cX'$.  

To find such a subset $\cX'$, \cite{AlonLMM19} model a randomized learner trained on an input sample~$S$ as a deterministic function $\cA(S):\cX\to\Delta(\{0,1\})$, where $\Delta(\{0,1\})$ denotes the space of probability distributions over~$\{0,1\}$. This space is identified with the interval $[0,1]$ via the standard mapping $\mu \mapsto \mu(1)$.  
Under this formulation, $\cA$ is comparison-based if the probability $\cA(S)(x)\in[0,1]$ depends only on the labels of the sorted input sample $S$ and the relative position of $x$ within~$S$.  

By applying the classical Ramsey theorem to a carefully-defined coloring of subsets of size $m+1$ of $\cX$, \cite{AlonLMM19} identify a large subset $\cX'\subseteq \cX$ on which $\cA$ is \emph{approximately} comparison-based in the following sense: there exists a comparison-based algorithm $\tilde{\cA}$ such that, for every sample $S$ and test point $x$ from $\cX'$,  
\[
	|\cA(S)(x) - \tilde\cA(S)(x)| \leq O\left(\frac{1}{m}\right).
\]
Having obtained $\cX'$, \cite{AlonLMM19} leverage the comparison-based nature of $\cA$ to attack privacy by identifying a family of indistinguishable distributions that satisfy a certain threshold property.  

A natural approach to extending the definition of comparison-based algorithms to $k$-list learners is to model a $k$-list learner trained on a sample $S$ as a mapping  
\[
	\cA(S): \cX \to \Delta\left({\cY \choose k}\right),
\]  
where the predicted list for each point $x$ depends only on its relative location with respect to the input sample. A $k$-list learner would then be considered approximately comparison-based if it is close to a comparison-based algorithm in total variation distance.  

Unfortunately, this approach is infeasible. The dimension of the simplex $\Delta\left({\cY \choose k}\right)$ is approximately $|\cY|^k$, which results in an enormous coloring space when applying the Ramsey argument. Moreover, the approach breaks down entirely when the label space $\cY$ is infinite, since the dimension of the simplex $\Delta\left({\cY \choose k}\right)$ would also be infinite.  
To overcome this issue, we significantly relax the comparison-based property by requiring only that the marginal distributions  
\[
	\Pr_{h\sim \cA(S)}[y\in h(x)]
\]  
are comparison-based for every \( y \in \cY \). This reduces the relevant simplex's dimension from \( |\cY|^k \) to just \( |\cY| \).  
We further refine this reduction by observing that it suffices to require the marginal distributions  
$\Pr_{h\sim \cA(S)}[y\in h(x)]$
to be comparison-based only for the \( k+1 \) labels \( y\in K \) that witness the monotone shattering (see \Cref{def:MD}). This further reduces the dimension to \( k+1 \), making the analysis much more tractable.  

While these relaxations lose the intuitive algorithmic semantics of comparison-based algorithms, they significantly simplify the proof and make it feasible to implement in the context of $k$-list learnability.

\section{Preliminaries}\label{sec:preliminaries}

\subsection{List Learning}\label{subsec:preliminaries_learning}

\paragraph{Multi-labeled hypotheses and list learners.}
Let $\cX$ be a domain, let $\cY$ be a label space, and let~$k\in \bbN$. We denote by $\cY \choose  k$ the family of subsets of $\cY$ of size $k$. A \emph{$k$-multi-labeled hypothesis} is a function $h:\cX\to {\cY \choose  k}$. 
Define the $0-1$ loss function of a $k$-multi-labeled hypothesis $h\in {\cY \choose  k}^\cX$, on a labeled example $(x,y)\in \cX\times\cY$, as
\[l\big(h;(x,y)\big)=\1[y\notin h(x)].\]
The \emph{empirical loss} of $h$ with respect to a dataset ${S=\big((x_1,y_1),\ldots,(x_m,y_m)\big)\in (\cX\times\cY)^m}$ is 
\[\loss{S}{h}=\frac{1}{m}\sum_{i=1}^m l\big(h;(x_i,y_i)\big).\]
A dataset $S$ is \emph{realizable} by $h$  if $\loss{S}{h}=0$.

A $k$-list learner is a (possibly randomized) algorithm $\cA:(\cX\times\cY)^\star\to {\cY \choose  k}^\cX$. 
We model a randomized a $k$-list learner $\cA$ as a deterministic map $(\cX\times\cY)^\star\times \cX \to [0,1]^\cY$. Specifically, given an input sample $S$ and a test point $x\in \cX$, we define 
\[\cA_{S,x}(y)=\Pr_{h\sim\cA(S)}[y\notin h(x)] \text{ , where $y\in\cY$}.\]
Notice that when the label space $\cY$ is finite, $\sum_{y\in\cY}(1-\cA_{S,x}(y))=k$. Equivalently, $\sum_{y\in\cY}\cA_{S,x}(y)=|\cY|-k$.

\paragraph{List PAC learning.}
Given a distribution ~$\cD$ over $\cX\times\cY$, the \emph{population loss} of a $k$-multi-labeled hypothesis $h$ with respect to~$\cD$ is
\[\loss{\cD}{h}=\EEE{(x,y)\sim \cD}{l\big(h;(x,y)\big)}.\]

We say that an algorithm $\cA$ is an \emph{$(\alpha, \beta)$-accurate $k$-list learner} for a class $\cC\subset \cY^\cX$, with sample complexity $m$, if for every realizable distribution $\cD$, \[\Pr_{S \sim \cD^m}\left[ \loss{\cD}{\cA(S)}\geq\alpha\right]\leq\beta.\] 
Here, $\alpha$ is called the \emph{error} and~$\beta$ is called the \emph{confidence parameter}.
A class $\cC$ is \emph{PAC $k$-list learnable} if there exist vanishing $\alpha(m),\beta(m)\to 0$ and an algorithm $\cA$ such that for all $m$, $\cA$ is a $(\alpha(m),\beta(m))$-accurate $k$-list learner for~$\cC$ with sample complexity $m$.

\paragraph{List online learning.}
The framework of $k$-list online learning can be formulated as the following online game. At time $t$, an adversary presents an example $x_t\in\cX$. The learner outputs a list of predictions of size $k$, $p_t \sim \cA(S_{<t},x_t)$, where $S_{<t}=\big((x_1,y_1),\ldots,(x_{t-1},y_{t-1})\big)$. Then, the adversary reveals the true label $y_t\in\cY$, and the learner suffers loss $l(p_t;y_t)=\1[y_t\notin p_t]$. The goal of the learner is to minimize the expected number of mistakes 
\[M(T)=\EEE{p_t\sim \cA(S_{<t},x_t)}{\sum_{t=1}^T l(p_t;y_t)}.\]
A concept class $\cC\subset \cY^\cX$ is $k$-list online learnable if there exists a $k$-list online learner $\cA$ and a number $M$, such that for every realizable input sequence of length $T$, $M(T)\leq M$.
A class $\cC$ is $k$-list online learnable if and only if its $k$-Littlestone dimension is finite \cite{MoranSTY23}.

\paragraph{$k$-Littlestone dimension.}
The $k$-Littlestone dimension is a combinatorial parameter introduced by \cite{MoranSTY23} that captures mistake and regret bounds in the setting of list online learning. The definition of the {$k$-Littlestone} dimension uses the notion of \emph{mistake trees}. 
A mistake tree is a $b$-ary\footnote{A $b$-ary tree is a tree in which every internal vertex has exactly $b$ children.} decision tree whose vertices are labeled with instances from~$\cX$ and edges are labeled with labels from~$\cY$ such that each internal vertex has $b$ different labels on its $b$ outgoing edges. A root-to-leaf path in a mistake tree corresponds to a sequence of labeled examples $\big((x_1,y_1),\dots,(x_d,y_d)\big)$. 
The point $x_i$ is the label of the $i$'th internal vertex in the path, and $y_i$ is the label of its outgoing edge to the next vertex in the path.
We say that a class $\cC$ \emph{shatters} a mistake tree if every root-to-leaf path is realizable by $\cC$.
The \emph{$k$-Littlestone dimension} of~$\cC$, denoted $\LD{\cC}{k}$, is the largest number~$d$ such that there exists a complete $(k+1)$-ary mistake tree of depth~$d$ shattered by~$\cC$.
If $\cC$ shatters arbitrarily deep $(k+1)$-ary mistake trees then we write $\LD{\cC}{k}=\infty$. Note that the Littlestione dimension is equal to the $1$-Littlestone dimension $\LD{\cH}{}=\LD{\cH}{1}$.

\subsection{Differential Privacy}\label{subsec:preliminaries_DP}
We use standard definitions and notation from the differential privacy literature. For further background, see, e.g., the surveys \cite{DR14,Vadhan17survey}.
For two numbers $a$ and~$b$, denote $a\overset{\epsilon,\delta}{\approx}b$ if $a\leq e^\epsilon\cdot b +\delta$, and $b\leq e^\epsilon\cdot a +\delta$.
Two probability distributions $p,q$ are \emph{$(\epsilon,\delta)$-indistinguishable} if for every event $E$, $p(E)\overset{\epsilon,\delta}{\approx}q(E)$.

\begin{definition}
	Let $\epsilon,\delta\geq 0$. A randomized learning rule $\cA$ is $(\epsilon, \delta)$-differentially private if for
	every pair of training samples $S, S' \in (\cX\times\cY)^m$ 
	that differ on a single example, the distributions $\cA(S)$ and $\cA(S')$ are $(\epsilon,\delta)$-indistinguishable.
\end{definition} 

A basic property of differential privacy is that privacy is preserved under post-processing; it enables arbitrary data-independent transformations to differentially private outputs without affecting their privacy guarantees. 

\begin{proposition}[Post-Processing]\label{prop:dp_post_processing}
	Let $\cA:\cZ^m\to\cR$ be any $(\epsilon,\delta)$-differentially private algorithm, and let $f:\cR\to\cR'$ be an arbitrary randomized mapping. Then $f\circ\cA:\cZ^m\to\cR'$ is $(\epsilon,\delta)$-differentially private.
\end{proposition}

\paragraph{Empirical learners.}
Let $\cC$ be a class. An algorithm $\cA$ is \emph{$(\alpha,\beta)$-accurate empirical $k$-list learner} with sample complexity $m$ if for every realizable sample $S$ of size $m$, 
\[\Pr_{h\sim\cA(S)}[\loss{S}{h}\geq \alpha]\leq \beta.\]

\cite{BunNSV15} proved that any private PAC learner can be transformed into a private empirical learner, while the sample complexity is increased only by a multiplicative constant factor. 

\begin{lemma}[Lemma 5.9 in \cite{BunNSV15}]\label{lemma:reduction_to_empirical_learner}
	Suppose $\cA$ is an $(\epsilon, \delta)$-differentially private $(\alpha, \beta)$-accurate PAC $k$-list learner for an hypothesis class $\cC$ with sample complexity $m$. Then there is an $(\epsilon, \delta)$-differentially private $(\alpha, \beta)$-accurate empirical $k$-list learner for $\cH$ with sample complexity $n = 9m$.
\end{lemma}
We remark that \Cref{lemma:reduction_to_empirical_learner} was proved in \cite{BunNSV15} within the context of traditional PAC learning. However, the proof also applies to more general settings, including the setting of list learning.

\paragraph{Interior point problem.}
An algorithm $\cA : [n]^m \to [n]$ solves the interior point problem on $[n]$ with probability $1-\beta$ if for every input sequence $x_1\ldots x_m\in [n]$,
\[\Pr[\min x_i\leq \cA(x_1,\ldots x_m)\leq \max x_i]\geq 1-\beta\]
where the probability is taken over the randomness of $\cA$; the number of data points $m$ is called the sample complexity of $\cA$.

\begin{theorem}[Theorem 3.2 in \cite{BunNSV15}]\label{thm:lower_bound_ipp}
	Let $0<\epsilon< 1/4$ be a fix number and let $\delta(m)\leq 1/50m^2$. Then for every positive integer~$m$, solving the interior point problem on $[n]$ with probability at least~$3/4$ and with $(\epsilon,\delta(m))$-differential privacy requires sample complexity $m\geq \Omega(\log^\star n)$.
\end{theorem}

\subsection{General Definitions and Notations}

\paragraph{Iterated logarithm and tower function.}
The tower function $\twr_{(t)}(x)$ and the iterated logarithm $\log _{(t)}(x)$ are defined by the recursions
\begin{equation*}
	\twr_{(i)}(x)=\begin{cases}
	x & i=1\\
	2^{\twr_{(i-1)}(x)} & i>1
	\end{cases},
	\qquad
	\log_{(i)}(x)=\begin{cases}
	x & i=0\\
	\log({\log_{(i-1)}(x)}) & i>0
	\end{cases}.
\end{equation*}
Note that for all $t$, \big($\twr_{(t)}(\cdot)\big)^{-1}=\log_{(t-1)}(\cdot)$. Finally,
\[\log^\star(x)=\min\{t\mid \log_{(t)}(x)\leq 1\}.\]

\section{$k$-Littlestone Dimension versus $k$-Monotone Dimension}\label{sec:LDvsMD}
In this section we prove \Cref{thm:LD_vs_MD}. 

\paragraph{A class $\cC_L$ with $\LD{\cC_L}{k}=1$ and $\MD{\cC_L}{k}=\infty$.}
Let $k>1$. Let $\cM_k(\cX)$ denote the class of all monotone functions with $k+1$ labels over $\cX$, and set
$\cC_L=\cM_k(\mathbb{N})$. The $k$-monotone dimension of
$\cM_k(\mathbb{N})$ is infinite. We claim that its $k$-Littlestone dimension is $1$.
It is straightforward to verify that $\cM_k(\mathbb{N})$ shatters a $k$-Littlestone tree of depth~$1$. Therefore, it suffices to show that there exists an online $k$-list learner for $\cM_k(\mathbb{N})$ that makes at most $1$ mistake on any realizable input sequence. The learner's strategy is as follows.
Initially, for every test point $x \in \mathbb{N}$, the learner predicts the list
\[\{0,\ldots,\lfloor k/2\rfloor-1,\lfloor k/2\rfloor+1,\ldots,k\},\] 
until the first timestamp $t$ where it makes a mistake on the point $x_t$. From that point onward, the learner modifies its strategy as follows: 
\begin{itemize} 
	\itemsep0em 
	\item For test points $x \leq x_t$, it predicts the list $\{0,\ldots,k-1\}$. 
	\item For test points $x > x_t$, it predicts the list $\{1,\ldots,k\}$. 
\end{itemize}

This strategy ensures that the learner makes at most $1$ mistake on any realizable input sequence, and therefore, $\LD{\cM_k(\mathbb{N})}{k}=1$, as desired.

\medskip
To conclude, $\cM_k(\mathbb{N})$ is online $k$-list learnable. However, by \Cref{thm:sc_of_monotone_functions}, whose proof is outlined in \Cref{sec:PPL_implies_MD}, it is not DP PAC $k$-list learnable. This completes the third item of Corollary~\ref{maincor:PPLvsOLL}.
To complete the proof of \Cref{thm:LD_vs_MD}, it remains to show that there exist a class with finite $k$-monotone dimension, and infinite $k$-Littlestone dimension.  

\paragraph{A class $\cC_M$ with $\LD{\cC_M}{k}=\infty$ and $\MD{\cC_M}{k}=1$.}
Consider $\cX$ as the set of vertices of an infinite complete $(k+1)$-ary decision tree, where each vertex is labeled with a unique point $x \in \cX$, and every vertex has $k+1$ outgoing edges labeled with ${0, \ldots, k}$. 
Define a concept class $\cC \subset \{0, \ldots, k\}^\cX$, consisting of all concepts that realize exactly one infinite branch of the tree, and label every point $x \in \cX$ outside the branch with $0$.

It is straightforward to verify that the $k$-Littlestone dimension of $\cC$ is infinite. We claim, however, that the $k$-monotone dimension of $\cC$ is $1$. In fact, we prove an even stronger statement: $\cC$ does not contain all constant functions over two points.
Assume for contradiction that there exist two points $x_1, x_2 \in \cX$ such that $\cC_n$ realizes all $k+1$ constant functions over $x_1$ and $x_2$.
Observe that one of these points must be a descendant of the other in the decision tree—otherwise, $\cC$ cannot realize, for example, the function $(1, 1)$ (or any other constant function that is not $(0,0)$).

Assume without loss of generality that $x_2$ is an $i$-descendant of $x_1$, meaning $x_2$ belongs to the $i$-th subtree emanating from $x_1$. In this case, $\cC$ cannot realize, for example, the function $(i+1, i+1)$. 
This contradiction establishes that the $k$-monotone dimension of $\cC$ is indeed~$1$.

\section{Private $k$-List-Learnability Implies Finite $k$-Littlestone Dimension}\label{sec:proof_of_thm_A}

In this section we prove \Cref{thm:DP_implies_LD}.

We start by introducing some notations and definitions which will be used in the proof.
Let~$T$ be a $b$-ary tree. We assume an order on the outgoing edges of each internal vertex $v$, labeling them with $ \{0, 1, \ldots, b-1\} $. We refer to the vertex connected to~$v$ by the edge labeled $i$ as the~$i$-th child of~$v$. An \emph{$m$-chain} is a subset of vertices of size $m$ that form a chain with respect to the natural partial order induced by $T$, that is $v<u$ if $u$ is a descendant of $v$. 
The \emph{type} of a chain $C=\{v_1<\ldots<v_m\}$ is a tuple $\vec{t}=\vec{t}(C)\in \{0,\ldots,b-1\}^{m-1}$ where $\vec{t}_i=j$ if and only if $v_{i+1}$ is a $j$'s descendant of $v_i$.
In other words, the type encodes the immediate turns made on the path from~$v_1$ to $v_m$. 
See e.g. \Cref{fig:fig2}.

\begin{figure}[!htb]
	\centering
	\begin{forest}
		[{$\emptyset$}, rectangle, fill=blue ,opacity=0.4
			[{$0$}, rectangle, fill=blue ,opacity=0.4, edge label={node[midway, above, font=\tiny, color=magenta] {\textcolor{blue}{0}}}, edge={black, very thick, dashed}
				[{$00$}, edge label={node[midway, left, font=\tiny, color=magenta]{0}}]
				[{$\underset{\vdots}{01}$}, edge label={node[midway, right, font=\tiny, color=magenta]{1}}]
				[{$02$}, rectangle, fill=blue ,opacity=0.4, edge label={node[midway, right, font=\tiny, color=magenta]{\textcolor{blue}{2}}}, edge={black, very thick, dashed}]
			]
			[{$1$}, rectangle, fill=red, opacity=0.4, edge label={node[midway, right, font=\tiny, color=magenta]{1}}
				[{$10$}, edge label={node[midway, left, font=\tiny, color=magenta]{\textcolor{blue}{0}}}, edge={black, very thick, dashed}
					[{$100$}, rectangle, fill=red, opacity=0.4, edge label={node[midway, left ,font=\tiny, color=magenta]{0}}
						[{$1000$}, edge label={node[midway, left, font=\tiny, color=magenta ]{0}}]
						[{$\underset{\vdots}{1001}$}, edge label={node[midway, left, font=\tiny, color=magenta ]{1}}]
						[{$1002$}, rectangle, fill=red, opacity=0.4, edge label={node[midway, right, font=\tiny, color=magenta ]{\textcolor{blue}{2}}}, edge={black, very thick, dashed}]
					]
					[{$\underset{\vdots}{101}$}, edge label={node[midway, left , font=\tiny, color=magenta]{1}}]
					[{$102$}, edge label={node[midway, right, font=\tiny, color=magenta]{2}}]
				]
				[{$\underset{\vdots}{11}$}, edge label={node[midway, right, font=\tiny, color=magenta]{1}}]
				[{$12$}, edge label={node[midway, right,font=\tiny, color=magenta]{2}}]
			]
			[{$2$}, edge label={node[midway, above, font=\tiny, color=magenta]{2}}
				[{$20$}, edge label={node[midway, left, font=\tiny, color=magenta]{0}}]
				[{$\underset{\vdots}{21}$}, edge label={node[midway, right, font=\tiny, color=magenta]{1}}]
				[{$22$}, edge label={node[midway, right, font=\tiny, color=magenta]{2}}]
			]
		]
	\end{forest}
	\caption{The $3$-chain 
		$C=\{\emptyset,\color{blue}0\color{black},0\color{blue}2\color{black}\}$ is of type $\vec t(C)=(0,2)$.
		The 3-chain $\Tilde{C}=\{1,1\color{blue}0\color{black}0,100\color{blue}2\color{black}\}$ has the same type as $C$.
	}
	\label{fig:fig2}
\end{figure}

Assume now that $T$ is a Littlestone tree, with internal vertices from a domain $\cX$, and edges from a label space $\cY$.
To ease the notations, from this point we will assume that the $b$ distinct labels on the outgoing edges of each internal vertex are $\{0, 1 , \ldots, b-1\}$ where by writing $``i"$, we mean the label on the $i$'th outgoing edge.\footnote{Note that in the case where $|\cY|>k+1$, the probabilities of the learner labeling $x$ as the labels on the $k+1$ edges, does not necessarily sum to $k$. However, every $k$-list learner can be converted to $k$-list learner for which this sum is $k$ (while maintaining utility and privacy), by a simple post-processing step: if the learner outputs a hypothesis that predicts a label different from the labels on the edges, replace it with one of them.}

We identify $(m+1)$-chains in $T$ with $T$-realizable samples of size~$m$ as follows.
An $(m+1)$-chain  $C=\{x_1<x_2<\ldots<x_{m+1}\}$ in $T$ corresponds to a sample $S=\big((x_1,y_1),\ldots,(x_m,y_m)\big) \in (\cX\times\cY)^m$, such that $y_i$ satisfies that $x_{i+1}$ belongs to the $y_i$'s subtree emanating from $x_i$. 
Namely, $(y_1,\ldots,y_m)$ is the type of $C$, $\vec t(C)$. 
An input sequence $S=\big((x_1,y_1),\ldots,(x_m,y_m)\big)$ is $T$-realizable if $S$ is realizable by a branch of $T$.

Let $S=\big((x_1,y_1),\ldots,(x_m,y_m)\big)$ be a sample that is realizable by $T$, and assume for convenience that $S$ is ordered, i.e.\ $x_1<x_2<\ldots<x_m$. An instance~$x$ is \emph{compatible} with $S$ if there exists a branch in $T$ that realizes $S$ and contains $x$. In such a case denote
\[S^{+x}\coloneqq S\cup(x,y_x),\]
where $y_x\in\cY$ is a label such that $S\cup(x,y_x)$ is $T$-realizable.
Note that if $x$ appears earlier on the branch than $x_m$, then there is a unique such $y_x$.
In the complementing case, when $x$ appears after $x_m$, 
there are $b$ different choices for $y$ such that $S\cup(x,y)$ is $T$-realizable. In that case we arbitrarily pick $S^{+x}=S\cup(x,y_x)$ for the smallest such $y$, and therefore $S^{+x}$ is well defined.
The \emph{location} of $x$ in $S$ is \[\mathtt{loc}_S(x)\coloneqq \max\{i\mid x_i<x\}.\] 
If $x<x_1$ then define $\mathtt{loc}_S(x)\coloneqq0$.

Recall, given a compatible test point $x\in\cX$ 
\[\cA_{S,x}(y_x)=\Pr_{h\sim\cA(S)}[y_x\notin h],\]
which is the loss of $\cA(S)$ on the point $(x,y_x)$.

The reduction to algorithms with comparison-based loss is achieved using a Ramsey theorem for trees. This theorem guarantees that for any given algorithm and a sufficiently deep Littlestone tree, there exists a deep enough subtree in which the algorithm has a comparison-based loss with respect to this subtree.

\begin{definition}[Subtree]\label{def:subtree}
	Let $T$ be a $b$-ary tree. Define a \emph{subtree} of $T$ by induction on its depth~$d$. All vertices of $T$ are subtrees of $T$ of depth $d=0$. For $d \ge 1$ a subtree
	of depth $d$ is obtained from an internal vertex $v$ of $T$ and $b$ subtrees of depth $d-1$ each rooted at a different child of $v$.
\end{definition}

\begin{theorem} [Ramsey theorem on trees. Theorem C in \cite{FioravantiHMST24ramsey}]\label{thm:ramsey_trees}
	For all $d,b,c,m$ there exists 
	\[n \le \twr_{(m)}(5\cdot b^{m-2}dc^{b^{m-1}}\log c)\] 
	such that, for every coloring of $m$-chains in the complete $b$-ary tree of depth $n$ with $c$ colors, there exists a $b^{m-1}$-chromatic complete subtree of depth $d$ (i.e.\ its $m$-chains are colored with at most $b^{m-1}$ colors). 
	Furthermore, the obtained subtree is \emph{type-monochromatic}, in the sense that if two $m$-chains are of the same type then they are colored with the same color.
\end{theorem}

We note that \cite{FioravantiHMST24ramsey} stated \Cref{thm:ramsey_trees} only for binary trees ($b=2$), but the same proof applies to general $b$-ary trees, with minor necessary modifications. For completeness, we provide the proof for general arity in \Cref{sec:ramsey_theorem__for_trees}.

\begin{definition}[Approximately Comparison-Based Loss]\label{def:rand_comparison_alg}
	Let $T$ be a $(k+1)$-ary mistake tree.
	A (randomized) $k$-list algorithm~$\cA$, defined over input samples of size $m$, has $\gamma$-comparison-based loss on $T$ if the following holds.
	There exist numbers~$p_{\vec t,i}\in[0,1]$ for $\vec t\in \{0,1\}^{m+1}$ and $i\in \{0,\ldots,m\}$
	such that for every input sample $S$ of size $m$ realizable by $T$, and for every $x$ compatible with $S$,
	\begin{equation*}\label{eq:comparison_alg_condition}
		\rvert \cA_{S,x}(y_x)-p_{\vec t,i}\lvert \leq \gamma,
	\end{equation*}
	where $\vec t=\vec t(S^{+x})$ is the type of the sample $S^{+x}$, $y_x$ is the label that satisfies $S^{+x} = S\cup(x,y_x)$, and $i=\mathtt{loc}_S(x)$ is the location of $x$ in $S$.
\end{definition}

Referring to the label $y_x$ as the correct label for $x$, this property of $\cA$ essentially states that the probability of $\cA$ making an error on a fresh example $x$ depends only on the type of the input sequence $S$ and the relative position of $x$ in $T$ with respect to $S$, up to a $\gamma$-approximation factor.

As a consequence of the Ramsey theorem for trees, it turns out that every algorithm can be reduced to an algorithm that has approximately comparison-based loss.

\begin{lemma}[Every Algorithm has an Approx.\ Comparison-based Loss on a Large Subtree]\label{lemma:every_alg_is_comparison_based_somewhere}
	Let $\cA$ be a (possibly randomized) $k$-list algorithm
	that is defined over input samples of size $m$ over a domain $\cX$, 
	and let $T$ be a $(k+1)$-ary mistake tree of depth $n$ whose vertices are labeled by instances from $\cX$. Then, there exists a subtree $T'$ of $ T$ of depth $\frac{\log_{(m+1)}(n)}{2^{a(k+1)^{m+1} m\log m}}$, where $a<24$ is a universal numerical constant, such that $\cA$ has $\left(\frac{1}{100m}\right)$-comparison-based loss on $T'$.
\end{lemma}

\medskip The upcoming lemma, in conjunction with \Cref{lemma:every_alg_is_comparison_based_somewhere} and \Cref{lemma:reduction_to_empirical_learner}, implies \Cref{thm:DP_implies_LD}, as we will prove shortly. 

\begin{lemma}\label{lemma:SC_of_CB_alg}[Sample Complexity for Privately Learning Trees]
	Let $T$ be a $(k+1)$-ary mistake tree of depth $n$ and let $\cA$ be a $k$-list algorithm defined over input samples of size~$m$. Assume that
	\begin{enumerate}
		\item $\cA$ is $(\epsilon,\delta(m))$-differentially private for some $\epsilon\leq\log\left(\frac{400k^2+1}{400k^2}\right)$, and $\delta(m)\leq \frac{1}{200k^2m^2}$.
		\item $\cA$ has $\left(\frac{1}{100m}\right)$-comparison-based loss on $T$.  
		\item \(\mathcal{A}\) is an \((\alpha, \beta)\)-accurate empirical $k$-list learner for \(T\) \footnote{Here, by empirical learner for $T$, we mean an empirical learner with respect to input samples that are realizable by (a branch of) \(T\).}, where \(\alpha = \beta = \frac{k!}{10^4(k+1)^{k+2}}\). 
	\end{enumerate}
	Then, $m=\Omega(\log^\star n)$.
	     
\end{lemma}

\paragraph{Proof of \Cref{thm:DP_implies_LD}.} 
Finally, \Cref{thm:DP_implies_LD} follows from \Cref{lemma:reduction_to_empirical_learner,lemma:SC_of_CB_alg,lemma:every_alg_is_comparison_based_somewhere}.
Due to its technical nature, the full proof appears in \Cref{sec:add_proof_thm:DP_implies_LD}.
Therefore, it is left to prove \Cref{lemma:SC_of_CB_alg,lemma:every_alg_is_comparison_based_somewhere}.

\subsection{Proof of \Cref{lemma:every_alg_is_comparison_based_somewhere}}

\begin{proof}
	Define a coloring of $(m+2)$-chains of $T$ as follows.
	Let $C=\{x_1<\ldots<x_{m+2}\}$ be an $(m+2)$-chain in $T$. 
	Recall that $C$ corresponds to a sample $S=\big((x_1,y_1),\ldots,(x_{m+1},y_{m+1})\big)$ of size $m+1$ where $(y_1,\ldots,y_{m+1})=\vec t(C)$ is the type of $C$.
	For each $i\in\{1,\ldots, m+1\}$, let $S^{-i}$ denote the sample $S\setminus(x_i,y_i)$. 
	Set~$q_i(C)$ to be the fraction of the form $\frac{r}{100m}$ that is closest to $A_{S^{-i},x_i}(y_{x_i})$ (in case of ties pick the smallest such fraction). The color assigned to $C$ is the list $(q_1(C),\ldots,q_{m+1}(C))$.
	\footnote{Notice that the color assigned to \(C\) depends only weakly on the last vertex \(x_{m+2}\) via the label \(y_{m+1}\). We find it more convenient and less cumbersome to increase the size of the chain by one rather than keeping track of the labels.}

	Therefore, the total number of colors is at most $c\coloneqq (100m+1)^{m+1}$. By Ramsey theorem for trees (\Cref{thm:ramsey_trees}) there exists a subtree $T'$ that is type-monochromatic with respect to the above coloring of depth
	\begin{align*}
		d\geq & \frac{\log_{(m+1)}(n)}{5 \cdot (k+1)^m c^{(k+1)^{m+1}}\log c}                                         \\
		=     & \frac{\log_{(m+1)}(n)}{5\cdot (k+1)^m (100m+1)^{(m+1)(k+1)^{m+1}}(m+1)\log(100m+1)}                   \\
		=     & \frac{\log_{(m+1)}(n)}{2^{\log 5+m\log(k+1)+\log(100m+1)(m+1)(k+1)^{m+1}+\log(m+1)+\log\log(100m+1)}} \\
		\geq  & \frac{\log_{(m+1)}(n)}{2^{a(k+1)^{m+1} m\log m}},                                                     
	\end{align*}
	where $1<a<24$ is a universal numerical constant.
	For every possible type $\vec t\in\{0,1\}^{m+1}$ and $i\in\{0,\ldots,m\}$, set $p_{\vec t,i}$ to be $q_{i+1}(C)$ where $C$ is any $\vec t$-typed $(m+2)$-chain in $T'$. Note that $p_{\vec t,i}$ is well defined since $T'$ is type-monochromatic.
	It is straightforward to verify that $\cA$ has $\left(\frac{1}{100m}\right)$-comparison-based loss on $T'$ with respect to $\{p_{\vec t,i}\}$, as wanted. 
\end{proof}

\subsection{Proof of \Cref{lemma:SC_of_CB_alg}}\label{sec:proof_of_reduction_to_ipp_lemma}

The proof of \Cref{lemma:SC_of_CB_alg} follows the same procedure as the reduction from the interior point problem given by \cite{FioravantiHMST24ramsey}. 
The reduction is carried out by constructing an algorithm $\tilde \cA$ designed to solve the interior point problem. 
For the sake of the upcoming analysis of the algorithm, we assume that the input points to the interior point problem are not too close to each other, as justified by the following lemma, which its proof is deferred to \Cref{sec:add_proof_lemma:rescaling_ipp}.

\begin{lemma}\label{lemma:rescaling_ipp}
	Let $\cA:[n]^m\to[n]$ be an $(\epsilon,\delta)$-differentially private algorithm, and let $C(n)\leq \log ^2 n$. Assume that for every input sequence $x_1,\ldots,x_m$ such that $min_{i\neq j}|x_i-x_j|\geq C(n)$,
	\[\Pr[\min x_i\leq \cA(x_1,\ldots,x_m)\leq \max x_i]\geq\frac{3}{4}.\]
	Then, $m\geq\Omega(\log^\star n)$.
\end{lemma}

Before describing the algorithm, we will introduce notation which will be used in this proof.
Given a branch $B$ in a tree and an example $z$ on $B$, we denote $b_z=y_z$ where $(z,y_z)\in B$. 

\begin{definition}[Almost-correct interval]
	Let $B$ be a branch in a $(k+1)$-ary mistake tree $T$, and let $h$ be a $k$-labeled hypothesis over the domain of $T$. Denote by $Z=(z_1<\ldots<z_l)$ a sequence of consecutive points of length $l$ in $B$. 
	We say that $Z$ is an almost-correct interval with respect to $h$, if
	\[\sum_{i=1}^l\1[b_{z_i}\notin h(z_i)]\leq \frac{1}{2(k+1)}\cdot l.\]
\end{definition}

Note that the guaranty of accumulated loss less than $\frac{1}{2(k+1)}\cdot l$ is better than a random guess, which has an accumulated loss of $\frac{1}{k+1}\cdot l$.
In other words, having an almost-correct interval implies that $h$ follows the branch with a relatively high probability (compared to a random guess).

\paragraph{Reduction from interior point problem.}
Let $T$ be a $(k+1)$-ary mistake tree of depth $n$, and $\cA$ be a $k$-list algorithm as in \Cref{lemma:SC_of_CB_alg}. Let $d_1\ldots d_{m}\in [n]$ be natural numbers, the input to the interior point problem. For convenience, assume they are ordered $d_1\leq \ldots\leq d_m$. Additionally, assume that $d_{i+1}-d_i>\log^2 n$ for all $1\leq i<m$.
Define the algorithm $\tilde{\cA}$ as follows.

\renewcommand{\algorithmicrequire}{\textbf{Input:}}
\renewcommand{\algorithmicensure}{\textbf{Output:}}

\begin{algorithm}[H]
	\caption{$\tilde \cA$ (Reduction from IPP, \cite{FioravantiHMST24ramsey})}\label{alg:ipp}
	\begin{algorithmic}
		\Require $d_1\ldots,d_m$.
		\State - Sample uniformly at random a branch $B\sim\mathtt{Branches}(T)$.
		\State - $S \gets \left((x_1,y_1),\ldots (x_m,y_m)\right)$, where $x_i$ is the point of depth $d_i$ in $B$, and $y_i=b_{x_i}$.
		\State - Sample $h\sim\cA(S)$.
		\State - Search for almost-correct intervals $(z_1\ldots z_l)$ with respect to $h$, of length $l= \lfloor\log ^2 n\rfloor$, where $n=\mathtt{depth}(T)$.
		\Ensure Output $\max \big\{\mathtt{depth}(z_1)\mid Z=(z_1\ldots z_l) \text{ is an almost-correct interval of length $l= \lfloor\log ^2 n\rfloor$}\big\}$.

		In other words, output the depth of the first point of the deepest almost-correct interval. If there are no almost-correct intervals of length $l= \lfloor\log ^2 n\rfloor$, return $n$.
	\end{algorithmic}
\end{algorithm}

From now on, when we refer to $Z$ as an almost-correct interval, we mean that it is an almost-correct interval with respect to $\cA(S)$, and that its length is $l= \lfloor\log ^2 n\rfloor$.

\begin{proposition}\label{proposition:reduction}
	Let $\cA$ be a $k$-list learner as in \Cref{lemma:SC_of_CB_alg}. Then, $\tilde \cA$ is $(\epsilon,\delta(m))$-differentially private, and with probability at least $\frac{3}{4}$ its output lies between $d_1$ and $d_m$.
\end{proposition}

\Cref{lemma:SC_of_CB_alg} is a direct corollary of \Cref{proposition:reduction,lemma:rescaling_ipp}.
In order to prove \Cref{proposition:reduction}, we first need to collect some lemmas.

Let $B$ and $S$ be the sampled branch and sequence, respectively, as described in \Cref{alg:ipp}. The key idea of the proof is to show that the probability of having an almost-correct interval below $S$ (i.e.\ $x_m<z_1$) is low, while the probability of having an almost-correct interval within $S$ (i.e.\ $x_1<z_1<z_l<x_m$) is high. This ensures that, with high probability, the output of $\tilde \cA$ is an interior point as desired.

\begin{lemma}[Almost-correct intervals below $S$ are unlikely]\label{lemma:deep_sequences_are_not_long_almost_correct}
	Let $S=\left((x_1,y_1),\ldots (x_m,y_m)\right)$ be the sequence from the description of Algorithm \ref{alg:ipp}. Then, the probability that there is an almost-correct interval $\left(z_1<\ldots <z_l\right)$ in the branch $B$ with $x_m<z_1$, is at most $n\cdot\exp{\left(-\frac{1}{8(k+1)}\lfloor\log^2 n\rfloor\right)}$.
\end{lemma}

\begin{proof}
	Let $l=\lfloor\log^2 n\rfloor$ and let $Z=(z_1<\ldots<z_l)$ be a sequence of consecutive examples on~$B$. Assume that $Z$ starts below $S$ on $T$, i.e.\ $x_{m}<z_1$.
	We can conceptualize the randomness of the reduction algorithm $\tilde \cA$
	in the following manner: Initially, a fair coin is independently tossed $d_m + 1$ times (recall that $d_m$ is the largest input). These coin tosses determine the first $d_m + 1$ turns of the branch $B$, which in turn determines the sample $S$, the input for $\cA$. Subsequently, the coin is independently tossed $n - d_m - 1$ more times, completing the determination of the branch $B$.
	This illustrates that, for every example $z$ that is below $S$, $\cA_S(z)$ is independent of $b_z$, and  $\Pr[b_{z}\not\in h(z)]=\frac{1}{k+1}$. Therefore, $\bbE\left[\sum_{i=1}^l \1[b_{z_i}\not\in h(z_i)]\right]=\frac{1}{k+1}l$,
	and by applying a standard Chernoff bound,
	\[\Pr_{B,h\sim \cA(S)}\left[ \sum_{i=1}^l \1[b_{z_i}\not\in h(z_i)] \leq \frac{1}{2(k+1)}l\right]\leq \exp{\left(-\frac{l}{8(k+1)}\right)}.\]
	By a simple union bound, we derive that the probability that there is an almost-correct interval below $S$ is at most $n\cdot\exp{\left(-\frac{l}{8(k+1)}\right)}= n\cdot\exp{\left(-\frac{1}{8(k+1)}\lfloor\log^2 n\rfloor\right)}$.
	   
\end{proof}

\begin{lemma}[Almost-correct intervals within $S$ are likely]\label{lemma:good_intervals_within_s_are_likely}
	Let $S=\left((x_1,y_1),\ldots (x_m,y_m)\right)$ be the sequence from the description of Algorithm \ref{alg:ipp}. Then, the probability that there is no almost-correct interval $\left(z_1<\ldots <z_l\right)$ on the branch $B$ with $x_1<z_1<z_l<x_m$, is at most 
	\[\exp\left(-\frac{\mu_k^2}{6(k+1)} m\right) + 2(k+1)\cdot\eta,\]
	where $\mu_k=\frac{(k+1)!}{(k+1)^{k+1}}$ and $\eta=2k(e^\epsilon-1+\delta(m))+\frac{2}{100m}+\frac{20}{10^4(k+1)}$.
\end{lemma}

The proof of Lemma~\ref{lemma:good_intervals_within_s_are_likely} consists of three parts. First, we show that with high probability, there exist $k+1$ consecutive examples in $S$, $x_i,\ldots,x_{i+k}$, with $k+1$ distinct labels (i.e.\, all $k+1$ possible labels), and with small error of $\cA$ on each of them. This argument relies on the fact that $\cA$ is an empirical learner for $T$.
In the second step, we use the properties of privacy and comparison-based error of $\cA$ to show that on each point $x_i<x<x_{i+k}$ on the branch $B$, the error of $\cA$ remains small.
Finally, applying Markov's inequality, we conclude that for any interval of length $l$ between $x_i$ and $x_{i+k}$, the probability to not be an almost-correct interval is small.

We call an example $x$ on the branch $B$ correct with parameter $\xi$ (or $\xi$-correct), if $\cA_{S,x}(b_x)<\xi$. We further say that a set of examples in $B$ is correct with parameter $\xi$ (or $\xi$-correct), if any example in the set is $\xi$-correct.

\begin{lemma}\label{lemma:k+1_good_points_in_S}
	Let $S=\left((x_1,y_1),\ldots (x_m,y_m)\right)$ be the sequence from the description of Algorithm~\ref{alg:ipp}, and let $\xi=\frac{20}{10^4(k+1)}$. Then, with probability $\geq 1-\exp\left(-\frac{\mu_k^2}{6(k+1)} m\right)$, there are $k+1$ consecutive points in $S$, $x_i,\ldots,x_{i+k}$, for some $i\in\{1,\ldots,m-k\}$, such that the two following properties hold.
	\begin{enumerate}
		\item $\{y_i,\ldots,y_{i+k}\}=\{1,\ldots,k+1\}$,
		\item $\{x_i,\ldots,x_{i+k}\}$ is $\xi$-correct, with $\xi=\frac{20}{10
		^4(k+1)}$,
	\end{enumerate}
	where $\mu_k=\frac{(k+1)!}{(k+1)^{k+1}}$.
\end{lemma}
\begin{proof}
	For simplicity, we divide the examples in $S$ into $\left\lfloor m/(k+1)\right\rfloor$ disjoint parts, each contains $k+1$ consecutive examples from $S$, and show that with probability of at least $1-\exp\left(-\frac{\mu_k^2}{6(k+1)} m\right)$, one of this parts satisfies the two properties.
	For this goal, we define for each $1\leq i\leq m-k$, a random variable $X_i$ as follows. $X_i$ gets the value $1$ if $x_i,\ldots,x_{i+k}$ have $k+1$ distinct labels and otherwise, it gets $0$.
	Since the branch $B$ is chosen uniformly at random, the expectation of $X_i$ is $\mu_k=\frac{(k+1)!}{(k+1)^{k+1}}$.
	Now look at all random variables $X_i$ with $i=1+c(k+1)$ for some integer $c$. The number of such variables is $\left\lfloor m/(k+1)\right\rfloor$, and they are IID. Let $X=\sum_{j=0}^{\left\lfloor m/(k+1)\right\rfloor-1}X_{1+j(k+1)}$ denote their sum. By Chernoff,
	\begin{align*}
		\Pr\left[ X\leq \bbE[X]-\left\lfloor \frac{m}{k+1}\right\rfloor\cdot \frac{\mu_k}{2}\right]
		  & \leq\exp\left(-2\left(\frac{\mu_k}{2}\right)^2\left\lfloor \frac{m}{k+1}\right\rfloor\right) \\
		  & \leq\exp\left(-\frac{\mu_k^2}{2} \frac{m}{3(k+1)}\right)                                     \\
		  & =\exp\left(-\frac{\mu_k^2}{6(k+1)} m\right).                                                 
	\end{align*}
	On the other hand, 
	\begin{align*}
		\Pr\left[ X\geq\bbE[X]-\left\lfloor \frac{m}{k+1}\right\rfloor\cdot \frac{\mu_k}{2}\right] 
		=\Pr\left[ X\geq\left\lfloor \frac{m}{k+1}\right\rfloor\cdot \frac{\mu_k}{2}\right]        
		\leq\Pr\left[ X\geq \frac{\mu_k}{6(k+1)} m\right].                                         
	\end{align*}
	
	Hence, with probability of at least $1-\exp\left(-\frac{\mu_k^2}{6(k+1)} m\right)$, the number of $k+1$ consecutive examples in $S$ with $k+1$ distinct labels is at least $\frac{\mu_k}{6(k+1)} m$.  
	
	Consider the case when there exist at least $\frac{\mu_k}{6(k+1)} m$ of $k+1$ consecutive examples in $S$ with $k+1$ distinct labels, and assume towards contradiction that the second property of $\xi$-correct does not hold for any such $k+1$ consecutive examples.
	$\cA$ is an $(\alpha,\beta)$-empirical list learner and therefore,
	\[\frac{1}{m}\sum_{i=1}^m\cA_{S,x_i}(y_i)\leq \alpha+\beta=\frac{2\mu_k}{10^4(k+1)^2}.\]
	But by the assumption above,
	\[\frac{1}{m}\sum_{i=1}^m\cA_{S,x_i}(y_i)\geq \frac{\mu_k}{6(k+1)}\xi=\frac{10\mu_k}{3\cdot 10^4(k+1)^2},\]
	which leads to a contradiction. Therefore, with probability of at least $1-\exp\left(-\frac{\mu_k^2}{6(k+1)} m\right)$, there are $k+1$ consecutive examples in $S$ satisfying the two properties as wanted.
\end{proof}

\begin{lemma}\label{lemma:small_error_for_point_inside}
	Let $S=\left((x_1,y_1),\ldots (x_m,y_m)\right)$ be the sequence from the description of Algorithm~\ref{alg:ipp}, and suppose that $x_i,\ldots,x_{i+k}$ are $k+1$ consecutive examples in $S$ satisfying the two properties from Lemma~\ref{lemma:k+1_good_points_in_S}. Let $x$ be an example in $B$ such that $x_i\leq x\leq x_{i+k}$. Then, 
	\[\cA_{S,x}(b_x)\leq\eta,\]
	where $\eta=2k(e^\epsilon-1+\delta(m))+\frac{2}{100m}+\frac{20}{10^4(k+1)}$.
\end{lemma}
\begin{proof}
	If $x\in\{x_i,\ldots,x_{i+k}\}$, then $x$ is correct with parameter $\frac{20}{10^4(k+1)}<\eta$ and we are done. Assume that $x\notin\{x_i,\ldots,x_{i+k}\}$. Since $x_i,\ldots,x_{i+k}$ have distinct labels, there is $x'\in \{x_i,\ldots,x_{i+k}\}$ with the same label $y$ as $x$. Without loss of generality, assume that $x'$ is below $x$, and change the sequence $S$ as follows. Take all the examples in $S$ that are between $x$ and $x'$, including $x'$, and replace them with examples on $B$ that are outside the interval $(x, x')$, to get a new sample $\tilde S$ of length $m$. 
	Notice that by this construction, 
	\[\mathtt{loc}_{\tilde S}(x)=\mathtt{loc}_{\tilde S}(x').\] 
	Further, $x$ and $x'$ have the same label $y$, so that the sequences $\tilde S^{+x}$ and $\tilde S^{+x'}$ have the same type.
	The algorithm $\cA$ has $\left(\frac{1}{100m}\right)$-comparison-based loss on $T$, so we can deduce that,
	\[\lvert\cA_{\tilde S,x}(y)-\cA_{\tilde S,x'}(y)\rvert<\frac{2}{100m}.\]
	Note that for every $\epsilon,\delta\geq 0$, if $a,b\in[0,1]$ satisfy $a\overset{\epsilon,\delta}{\approx}b$, then $|a-b|\leq e^\epsilon-1+\delta$. Indeed,
	\begin{align*}
		a-b\leq (e^\epsilon-1)\cdot b +\delta\leq e^\epsilon-1+\delta, \\
		b-a\leq (e^\epsilon-1)\cdot a +\delta\leq e^\epsilon-1+\delta. 
	\end{align*}
	Since $\cA$ is $(\epsilon,\delta)$-differentially private, and the sequence $\tilde S$ differs from the sequence $S$ by at most $k$ entries, we can transform from $S$ to $\tilde S$ by at most $k$ replacements of a single example at each time.
	Therefore,
	\begin{align*}
		|\cA_{S,x}(y)-\cA_{\tilde S,x}(y)|\leq k(e^\epsilon-1+\delta),   \\ 
		|\cA_{S,x'}(y)-\cA_{\tilde S,x'}(y)|\leq k(e^\epsilon-1+\delta). 
	\end{align*}
	Combining all together we get:
	\begin{align*}
		\cA_{S,x}(y) & \leq |\cA_{S,x}(y)-\cA_{\tilde S,x}(y)| + \lvert\cA_{\tilde S,x}(y)-\cA_{\tilde S,x'}(y)\rvert + |\cA_{\tilde S,x'}(y)-\cA_{S,x'}(y)| + \cA_{S,x'}(y) \\
		             & \leq 2k(e^\epsilon-1+\delta)+\frac{2}{100m}+\frac{20}{10^4(k+1)}.                                                                                     
	\end{align*}
\end{proof}

\begin{proof}[Proof of \Cref{lemma:good_intervals_within_s_are_likely}]
	Assume that $x_i,\ldots,x_{i+k}$ are $k+1$ consecutive examples in $S$ satisfying the two properties from \Cref{lemma:k+1_good_points_in_S}. Let $z_1<\ldots<z_l$ be an interval such that $x_i<z_1<\ldots<z_l<x_{i+k}$. By \Cref{lemma:small_error_for_point_inside},
	\[\mathop\bbE_{B,h\sim \cA(S)}\left[\sum_{i=1}^l\1[b_{z_i}\not\in h(z_i)]\right]\leq \eta\cdot l.\]
	By Markov,
	\[\Pr\left[\sum_{i=1}^l\1[b_{z_i}\not\in h(z_i)]\geq \frac{1}{2(k+1)}\cdot l\right]\leq 2(k+1)\cdot\eta.\]
	By \Cref{lemma:k+1_good_points_in_S}, the probability for having such $x_i,\ldots,x_{i+k}$ examples in at least $1-\exp\left(-\frac{\mu_k^2}{6(k+1)} m\right)$.
	All in all, we conclude that the probability of not having an almost-correct interval inside $S$ is less than $\exp\left(-\frac{\mu_k^2}{6(k+1)} m\right) + 2(k+1)\cdot\eta$.
\end{proof}

We turn to prove \Cref{proposition:reduction} and show that $\tilde \cA$ is $(\epsilon,\delta(m))$-differentially private, and with probability at least $3/4$ its output is an interior point.

\begin{proof}[Proof of \Cref{proposition:reduction}]
	
	\hfill
	\paragraph{Privacy:}
	Let \(D\) and \(D'\) be neighboring datasets. Consider the output distributions \(\tilde{\mathcal{A}}(D)\) and \(\tilde{\mathcal{A}}(D')\). We couple these distributions by selecting the same random branch \(B\) in the first step of the algorithm. Hence, the samples \(S\) and \(S'\) that are input to \(\mathcal{A}\) are also neighbors. Since \(\mathcal{A}\) is \((\epsilon, \delta)\)-DP, it follows that the distributions \(\mathcal{A}(S)\) and \(\mathcal{A}(S')\) are \((\epsilon, \delta)\)-indistinguishable. Since the outputs \(\tilde{\mathcal{A}}(D)\) and \(\tilde{\mathcal{A}}(D')\) are functions of \(\mathcal{A}(S)\) and \(\mathcal{A}(S')\), by post-processing (\Cref{prop:dp_post_processing}), it follows that \(\tilde{\mathcal{A}}(D)\) and \(\tilde{\mathcal{A}}(D')\) are also \((\epsilon, \delta)\)-indistinguishable. Hence, \(\tilde{\mathcal{A}}\) is \((\epsilon, \delta)\)-DP.
	    
	\paragraph{Utility:}
	By \Cref{lemma:deep_sequences_are_not_long_almost_correct}, the probability that there is an almost-correct interval below $S$ is at most 
	\[n\cdot\exp{\left(-\frac{1}{8(k+1)}\lfloor\log^2 n\rfloor\right)}.\]
	By \Cref{lemma:good_intervals_within_s_are_likely}, the probability that there is no almost-correct interval within $S$, is at most 
	\[\exp\left(-\frac{\mu_k^2}{6(k+1)} m\right) + 2(k+1)\cdot\eta,\]
	where $\mu_k=\frac{(k+1)!}{(k+1)^{k+1}}$ and $\eta=2k(e^\epsilon-1+\delta(m))+\frac{2}{100m}+\frac{20}{10^4(k+1)}$.
	Hence, the probability that $\tilde \cA$ does not output an interior point is as most
	\begin{align*}
		  & n\cdot\exp{\left(-\frac{\lfloor\log^2 n\rfloor}{8(k+1)}\right)} + \exp\left(-\frac{\mu_k^2}{6(k+1)} m\right) + 2(k+1)\cdot\eta                                                     \\
		= & n\cdot\exp{\left(-\frac{\lfloor\log^2 n\rfloor}{8(k+1)}\right)} + \exp\left(-\frac{\mu_k^2}{6(k+1)} m\right) + 4k(k+1)(e^\epsilon-1+\delta(m))+\frac{4(k+1)}{100m}+40\cdot10^{-4}. 
	\end{align*}
	Each one of the five summands is smaller than $1/20$, by the choice of $\epsilon=\log\left(\frac{400k^2+1}{400k^2}\right),\delta(m)=\frac{1}{200k^2m^2}$, 
	and for large enough
	\footnote{Note that, without loss of generality, we may assume that $m\geq 1/\alpha$ since $\cA$ is an $(\alpha, \beta)$-accurate learner for $T$. Additionally, the assumption that $n$ is large enough is concealed in the big $\Omega$ notation.} $m$ and $n$.
\end{proof}

\section{Private $k$-List-Learnability Implies Finite $k$-Monotone Dimension}\label{sec:proof_of_thm_B}

In this section, we prove \Cref{thm:sc_of_monotone_functions}. 

Let $\cM_k(\cX)$ denote the class of all monotone functions with $k+1$ labels over $\cX$. To prove \Cref{thm:sc_of_monotone_functions}, it suffices to establish the following lemma. This is because (i) any class with $k$-monotone dimension at least $n$ contains a copy of $\cM_k([n])$
(ii) any $k$-list learner can be converted to $k$-list learner for which its prediction list is among the $k+1$ labels from the definition of the $k$-monotone dimension (while maintaining utility and privacy), by a simple post-processing step: if the learner outputs a hypothesis that predicts a label outside this set, it is replaced with one of the $k+1$ labels.

\begin{lemma}\label{lemma:sc_of_monotone_fubctions}
	Let $k\geq1$. Let $\cA$ be an \textcolor{black}{$(1/200k(k+),1/200k(k+1))$}-accurate $k$-list learning algorithm for the class $\cM_k([n])$ with sample complexity $m$, satisfying $(\epsilon,\delta(m))$-differential privacy for $\epsilon=\textcolor{black}{0.1}$ and $\delta(m)\leq \textcolor{black}{\frac{1}{6(200km)^4\log^2(200km)}}$.
	Then the following bound holds:
	\begin{equation*}
		m = \Omega(\log^\star n).
	\end{equation*}
\end{lemma}

To prove \Cref{lemma:sc_of_monotone_fubctions}, we begin by introducing some notations and definitions that will be useful throughout the proof.

\paragraph{Introducing notations.} Given a linearly ordered domain $\cX$, a sequence $S=\big((x_1,y_1),\ldots,(x_m,y_m)\big)$ is \emph{ordered} if $x_1<x_2<\ldots<x_m$. Given an ordered sequence $S$ and a test point $x\in \cX$, the \emph{location} of $x$ in $S$ is
\[\mathtt{loc}_S(x)\coloneqq \max\{i\mid x_i<x\},\] 
and if $x\leq x_1$ then define $\mathtt{loc}_S(x)\coloneqq0$.
An ordered sequence $S=\big((x_1,y_1),\ldots,(x_m,y_m)\big)$ is \emph{increasing} if $y_1\leq\ldots \leq y_m$. An ordered sequence is \emph{balanced} if it is increasing, and every label out of $\cY=\{0,\ldots ,\ell-1\}$ appears the same amount of times. I.e.,  say $|S|=t\cdot \ell$ then  ${y_{i\cdot t+1}=\ldots =y_{(i+1)\cdot t}=i}$, for $i=0,\ldots ,\ell-1$.

Recall, given an input sample $S$ and a test point $x\in \cX$, we denote  
\[\cA_{S,x}(y)=\Pr_{h\sim\cA(S)}[y\notin h] \text{ , where $y\in\cY$}.\]
Next, we adapt the notion of \emph{comparison-based} algorithms, as introduced in \cite{AlonLMM19,FioravantiHMST24ramsey}, to the setting of $k$-list learners.
Roughly speaking, a $k$-list algorithm $\cA$ is comparison-based if the prediction of $\cA$ on a test point $x$ depends only on the labels of the sorted input sample $S$ and the position of $x$ inside of $S$. I.e.\ the algorithm makes all its decisions only based on how the elements of the input sample and the test point compare to each other, and not on their absolute values/locations.

\begin{definition}[Approximately comparison-based on balanced samples]
	Let $\cX$ be a linearly ordered domain, let $\cY=\{0,\ldots,\ell-1\}$ be the label space, and let $\gamma>0$. A (randomized) $k$-list learner $\cA$, defined over input samples of size $m$, is $\gamma$-comparison-based (CB) with respect to $\cX$ if the following holds.
	There exist vectors ${\vec p}^{(0)},\ldots , {\vec p}^{(m)}\in [0,1]^\ell$ such that for every increasing balanced input sequence $S\in (\cX\times \cY)^m$, and every test point $x\in \cX$,
	\[\lVert \cA_{S,x}-{\vec p}^{(i)}\rVert_\infty\leq \gamma,\]
	where $i=\mathtt{loc}_S(x)$. 
\end{definition}

The following two lemmas are key to proving \Cref{lemma:sc_of_monotone_fubctions}.

\begin{lemma}[Every algorithm is CB on a large subset]\label{lemma:every_alg_is_cb_somewhere}
	Let $\cA$ be a $k$-list learner that is defined over input samples of size $m$, over a linearly ordered domain $\cX$ with $|\cX|=n$, and a label space $\cY=\{0,\ldots, \ell-1\}$. Then, there exist $\cX'\subset \cX$ such that $\cA$ is $\left(\frac{1}{100km}\right)$-comparison-based on $\cX'$ and 
	\[|\cX'|\geq \frac{\log_{(m)}(n)}{2^{O(\ell\cdot m\log km)}},\]
	where the big $O$ notation hides a universal constant value.
\end{lemma}

\begin{lemma}[Lower-bounding the sample complexity of CB algorithms]\label{lemma:lower_bounding_sc}
	Let $\cA$ be a $k$-list learner that is defined over input samples of size $m$, over a linearly ordered domain $\cX$, and a label space $\cY=\{0,1,\ldots,k\}$. Assume that
	\begin{enumerate}
		\item $\cA$  is $(\epsilon,\delta(m))$ differentially private for  $\epsilon={0.1}, \delta(m)\leq \textcolor{black}{\frac{1}{6(200km)^4\log^2(200km)}}$.
		\item $\cA$ is $\left(\frac{1}{100km}\right)$-comparison-based on $\cX$.
		\item $\cA$ is $(\alpha,\beta)$-accurate empirical $k$-list learner for the class $\cM_k(\cX)$, with $\alpha={\frac{1}{200k(k+1)}},\beta={\frac{1}{200k(k+1)}}$.  
	\end{enumerate}
	Then, 
	$|\cX|\leq 2^{O((km)^2\log^2 (km))}$, where the big $O$ notation hides a universal constant value.
\end{lemma}

\begin{proof}[Proof of \Cref{lemma:sc_of_monotone_fubctions}]
	First, by \Cref{lemma:reduction_to_empirical_learner} we can assume that $\cA$ empirically learns the class of step-functions over $[n]$. Then, by \Cref{lemma:every_alg_is_cb_somewhere,lemma:lower_bounding_sc} there exists $\cX'\subset [n]$ such that $\cA$ is comparison based on $\cX'$, and 
	\[\frac{\log_{(m)}(n)}{2^{O(km\log (km))}}\leq|\cX'|\leq 2^{O((km)^2\log^2 (km))},\]
	therefore we have $\log_{(m)}(n)\leq 2^{c\cdot (km)^2\log^2 (km)}$ for some constant $c$. By taking iterated logarithm $t=\log^\star(2^{c\cdot (km)^2\log^2 (km)})=\log^\star(km)+O(1)$ times of both sides, we obtain $\log_{(m+t)}(n)\leq 1$, and therefore $\log^\star(n)\leq m+t=m+\log^\star(km)+O(1)$. As we treat $k$ as a constant, it implies that $m\geq \Omega(\log^\star n)$ as required.
\end{proof}

Therefore, it is left to prove \Cref{lemma:every_alg_is_cb_somewhere,lemma:lower_bounding_sc}.

\subsection{Proof of \Cref{lemma:every_alg_is_cb_somewhere}}

The key to proving \Cref{lemma:every_alg_is_cb_somewhere} is Ramsey theorem.

\begin{theorem}[Ramsey \cite{ErdosRado52}]\label{thm:ramsey}
	Let $s>t\geq 2$ and $q$ be integers, and let 
	\[N\geq \twr_{(t)}(3sq\log q).\]
	Then for every coloring of subsets of size $t$ of a universe of size $N$ using $q$ colors, there is a homogeneous subset of size $s$.
\end{theorem}

\begin{proof}[Proof of \Cref{lemma:every_alg_is_cb_somewhere}]
	Define a coloring of $(m+1)$-subsets of $\cX$ as follows. Let $X=\{x_1<x_2<\ldots<x_{m+1}\}$ be an $(m+1)$-subset of $\cX$. For each $i\leq m+1$, denote $X^{-i}=X\setminus\{x_i\}$ and $S^{-i}$ the balanced increasing sample on $X^{-i}$. 
	        
	Next, for every $i\leq m+1$ 
	observe $\cA_{S^{-i},x_i}=\big(\cA_{S^{-i},x_i}(0),\ldots,\cA_{S^{-i},x_i}(\ell-1)\big)$, where $\cA_{S^{-i},x_i}(j)=\Pr_{h\sim \cA({S^{-i}})}[j\notin h(x_i)]$. Set $p^i_j$ to be the fraction of the form $\frac{t}{100km}$ closest to $\cA_{S^{-i},x_i}(j)$ (in case of ties pick the smaller one). Now set ${\vec p}^{(i)}=(p^i_0,\ldots, p^i_{\ell-1})$. The color assigned to $X$ is the list $(\vec {p}^{(1)},\ldots, \vec {p}^{(m+1)})$.
	
	Therefore, the total number of colors is $(100km+1)^{\ell(m+1)}$.\footnote{The number of colors can be reduced to  ${(100km+1)^{(\ell-1)(m+1)}}$, because $p^{i}_{\ell-1}$ is a function of $p^i_0,\ldots, p^i_{\ell-2}$.}
	By applying \Cref{thm:ramsey} with $t\coloneqq m+1, q\coloneqq {(100m+1)^{\ell(m+1)}}$, and $N\coloneqq n$ there is an homogeneous subset $\cX'\subseteq \cX$ of size
	\begin{align*}
		|\cX'|\geq \frac{\log_{(m)}(n)}{3(100km+1)^{\ell(m+1)}\ell(m+1)\log(100km+1)}=\frac{\log_{(m)}(n)}{2^{O(\ell \cdot m\log (km))}} 
	\end{align*}
	such that all $(m+1)$-subsets of $\cX'$ have the same color. One can verify that $\cA$ is  $\left(\frac{1}{100km}\right)$-comparison-based on $\cX'$.
\end{proof}

\subsection{Proof of \Cref{lemma:lower_bounding_sc}}

\Cref{lemma:lower_bounding_sc} is a direct result of the following propositions.

\begin{proposition}\label{proposition:establishing_dist_for_packing}
	Let $\cA$ be as in \Cref{lemma:lower_bounding_sc} and let $n=|\cX|-m$. Then, there exist distributions $\cP_1,\ldots, \cP_n$, a number $c\in[0,1]$, and events $E_1,\ldots,E_n$, such that the following holds: 
	\begin{enumerate}
		\item [(i)] for every $i,j$, $\cP_i$ and $\cP_j$ are $(\epsilon,\delta(m))$-indistinguishable for $\epsilon=0.1$ and $\delta(m)\leq \textcolor{black}{\frac{1}{6(200km)^4\log^2(200km)}}$; and
		\item [(ii)] for every $i,j$,  \begin{equation*}\label{eq:threshold_prop}
		      \cP_i(E_j)=\begin{cases}
		      \leq c-\gamma & j<i \\
		      \geq c+\gamma & j>i,
		\end{cases}
		\end{equation*}
		where $\gamma=\frac{1}{200km}$.
	\end{enumerate}
	       
\end{proposition}

\begin{proposition}[Packing]\label{proposition:packing}
	Let $c\in[0,1], \gamma\leq\frac{1}{2}$, and
	let $\cP_1,\ldots, \cP_n$ be probability measures such that for all $i,j$, $\cP_i$ and $\cP_j$ are $(\epsilon,\delta)$-indistinguishable, for $\epsilon\leq0.1$ and 
	$\delta\leq\frac{1}{6\gamma^{-4}\log^2\left(\frac{1}{\gamma}\right)}$.
	Assume there exist events $E_1,\ldots, E_n$ such that for every $i,j$
	\begin{equation}\label{eq:threshold_prop}
		\cP_i(E_j)=\begin{cases}
		\leq c-\gamma & j<i \\
		\geq c+\gamma & j>i.
		\end{cases}
	\end{equation}
	Then, $n\leq 2^{\gamma^{-2}\log^2\left(\gamma^{-1}\right)}$.
\end{proposition}

The proofs of \Cref{proposition:establishing_dist_for_packing,proposition:packing} are deferred to \Cref{sec:add_proof_prop:packing,sec:add_proof_prop:establishing_dist_for_packing}.

\bibliographystyle{alpha}
\newpage
\bibliography{references}

\newpage

\appendix

\crefalias{section}{appendix} 

\section{Additional Proofs}\label{sec:add_proofs}

\subsection{Proof of \Cref{thm:DP_implies_LD}}\label{sec:add_proof_thm:DP_implies_LD}

\begin{theorem*}[Restatement of \Cref{thm:DP_implies_LD}]
	Let $k\geq1$. Let $\cC\subset\cY^\cX$ be a concept class with $k$-Littlestone dimension ${\LD{\cC}{k}\geq d}$, 
	and let $\cA$ be an \textcolor{black}{$\left(\frac{k!
		}{10^4(k+1)^{k+2}},\frac{k!
		}{10^4(k+1)^{k+2}}\right)$}-accurate $k$-list learning algorithm for $\cC$ with sample complexity $m$, satisfying $(\epsilon,\delta(m))$-differential privacy for $\epsilon=\textcolor{black}{\log\left(\frac{400k^2+1}{400k^2}\right)}$ and ${\delta(m)\leq \textcolor{black}{\frac{1}{200k^2m^2}}}$.
	Then, the following bound holds:
	\begin{equation*}
		m = \Omega(\log^\star d), 
	\end{equation*}
	where the $\Omega$ notation conceals a universal numerical multiplicative constant.
\end{theorem*}

\Cref{thm:DP_implies_LD} follows from \Cref{lemma:reduction_to_empirical_learner,lemma:SC_of_CB_alg,lemma:every_alg_is_comparison_based_somewhere}. 

\begin{proof}
	Let $\cH$ be a concept class over an arbitrary label domain $\cY$, and assume that $\cH$ shatters a $k$-Littlestone tree $T$ of depth $d$. Let $\cA$ be any 
	$(\epsilon,\delta(m))$-differentially private $k$-list learner for~$\cH$, with $\epsilon,\delta(m)$ as in \Cref{thm:DP_implies_LD}.
	By \Cref{lemma:reduction_to_empirical_learner}, we can further assume that $\cA$ is an $(\alpha,\beta)$-accurate empirical learner for $\cH$, for  $\alpha=\beta=\frac{k!}{10^4(k+1)^{k+2}}$, as the sample complexity of a private empirical learner increases only by a multiplicative constant factor.
	By \Cref{lemma:every_alg_is_comparison_based_somewhere}, there exists a subtree~$T'$ of $T$, of depth $\frac{\log_{(m+1)}(d)}{2^{a(k+1)^{m+1} m\log m}}$ for some universal numerical constant $a<24$, such that $\cA$ has $\left(\frac{1}{100m}\right)$-comparison-based loss on $T'$.
	Finally, by \Cref{lemma:SC_of_CB_alg} we conclude that 
	\[m\geq \Omega\left(\log^\star\left(\frac{\log_{(m+1)}(d)}{2^{a(k+1)^{m+1} m\log m}}\right)\right).\]
	Let $t=\log^\star(d)$ and suppose $m\leq \textcolor{black}{\frac{t}{16\log(k+1)}}$ (else $m=\Omega(\log^\star d)$ as we treat $k$ as constant, and we are done). We claim that $\log^\star\left(\frac{\log_{(m+1)}(d)}{2^{a(k+1)^{m+1} m\log m}}\right)=\Omega(\log^\star d)$, and therefore $m\geq\Omega(\log^\star(d))$, which concludes the proof.
	Note that by the definition of the $\log^\star$ function, $\twr_{(t)}(1)<d\leq \twr_{(t+1)}(1)$.
	The claim follows from the following calculation: 
	\begin{align*}
		\log^\star\left(\frac{\log_{(m+1)}(d)}{2^{a(k+1)^{m+1} m\log m}}\right) & =1+\log^\star\left(\log_{(m+2)}(d)-a(k+1)^{m+1} m \log m\right)\tag{definition of $\log^\star$}                                                   \\
		                                                                        & \geq 1+\log^\star\left(\twr_{(t-(m+2))}(1)-a(k+1)^{m+1} m \log m\right)\tag{$d>\twr_{(t)}(1)$}                                                    \\
		                                                                        & \geq1+\log^\star\left(\twr_{(t/2)}(1)-a(k+1)^{m+1} m \log m\right)\tag{holds for $t\geq5$ since $m\leq \textcolor{black}{\frac{t}{16\log(k+1)}}$} \\
		                                                                        & \geq1+\log^\star\left(\twr_{(t/2)}(1)-2^{t/2}\right)\tag{$\forall m: ~24(k+1)^{m+1}m\log m\leq (k+1)^{8m} \leq2^{t/2}$ }                          \\
		                                                                        & =1+\log^\star\left(\twr_{(t/2)}(1)\right)\tag{holds for $t\geq 10$, see justification below}                                                      \\
		                                                                        & =t/2.\tag{definition of $\log^\star$}                                                                                                             
	\end{align*}
	Therefore, for large enough $d$, $\log^\star\left(\frac{\log_{(m+1)}(d)}{2^{a2^{m\log(k+1)} m\log m}}\right)\geq \frac{1}{2}\log^\star d$, as desired.
	It is left to justify the second-to-last equality. It is enough to show that ${\twr_{(x)}(1)-2^x>\twr_{(x-1)}(1)}$ for $x\geq5$. And indeed $\twr_{(x)}(1)-\twr_{(x-1)}(1)\geq\frac{1}{2}\twr_{(x)}(1)\geq 2^x$ for $x\geq 5$.
\end{proof}

\subsection{Proof of \Cref{lemma:rescaling_ipp}}\label{sec:add_proof_lemma:rescaling_ipp}

\begin{lemma*}[Restatement of \Cref{lemma:rescaling_ipp}]
	Let $\cA:[n]^m\to[n]$ be an $(\epsilon,\delta)$-differentially private algorithm, and let $C(n)\leq \log ^2 n$. Assume that for every input sequence $x_1,\ldots,x_m$ such that $min_{i\neq j}|x_i-x_j|\geq C(n)$,
	\[\Pr[\min x_i\leq \cA(x_1,\ldots,x_m)\leq \max x_i]\geq\frac{3}{4}.\]
	Then, $m\geq\Omega(\log^\star n)$.
\end{lemma*}

\begin{proof}
	By rescaling, $\cA$ solves the interior point problem on a domain of size~$\frac{n}{C(n)}$. Therefore, by \Cref{thm:lower_bound_ipp}, the sample complexity satisfies $m\geq \Omega\left(\log^\star\left(\frac{n}{C(n)}\right)\right)= \Omega\left(\log^\star\left( \frac{n}{\log^2 n}\right)\right)= \Omega(\log^\star n)$.
	Indeed, the second equality holds since by the definition of the $\log^\star$ function, $\twr_{(t+2)}(1)\geq2^{\frac{n}{\log^2 n}}$, where $t=\log^\star\left(\frac{n}{\log^2 n}\right)$.
	Furthermore, $2^{\frac{n}{\log^2 n}}\geq n$ for large enough $n$.
	Therefore, again by definition, $t+1\geq \log^\star n$ for large enough $n$, which implies that  $\log^\star\left(\frac{n}{\log^2 n}\right)\geq \frac{1}{2}\log^\star n$ for large enough $n$.
\end{proof}

\subsection{Proof of \Cref{proposition:establishing_dist_for_packing}}\label{sec:add_proof_prop:establishing_dist_for_packing}

\begin{proposition*}[Restatement of \Cref{proposition:establishing_dist_for_packing}]
	Let $\cA$ be as in \Cref{lemma:lower_bounding_sc} and let $n=|\cX|-m$. Then, there exist distributions $\cP_1,\ldots, \cP_n$, a number $c\in[0,1]$, and events $E_1,\ldots,E_n$, such that the following holds: 
	\begin{enumerate}
		\item [(i)] for every $i,j$, $\cP_i$ and $\cP_j$ are $(\epsilon,\delta(m))$-indistinguishable for $\epsilon=0.1$ and $\delta(m)\leq \textcolor{black}{\frac{1}{6(200km)^4\log^2(200km)}}$; and
		\item [(ii)] for every $i,j$,  \begin{equation*}\label{eq:threshold_prop}
		      \cP_i(E_j)=\begin{cases}
		      \leq c-\gamma & j<i \\
		      \geq c+\gamma & j>i,
		\end{cases}
		\end{equation*}
		where $\gamma=\frac{1}{200km}$.
	\end{enumerate}
	       
\end{proposition*}

The key for proving \Cref{proposition:establishing_dist_for_packing} is the following proposition.

\begin{proposition}\label{proposition:jump}
	Let $\cA$ be as in \Cref{lemma:lower_bounding_sc} and let ${\vec p}^{(0)},\ldots , {\vec p}^{(m)}\in [0,1]^{k+1}$ be the vectors promised by the definition of comparison-based algorithms. Then, there exists $0<i\leq m$ such that
	\[\lVert \vec{p}^{(i)}-\vec{p}^{(i-1)}\rVert_\infty\geq {\frac{3}{100km}}.\]
\end{proposition}

\begin{proof}[Proof of \Cref{proposition:establishing_dist_for_packing}]
	Let $i$ be the index promised in \Cref{proposition:jump}. I.e.\ $\lVert \vec{p}^{(i)}-\vec{p}^{(i-1)}\rVert_\infty\geq \frac{3}{100km}$. W.l.o.g assume $|\vec p^{(i)}_0-\vec p^{(i-1)}_0|\geq \frac{3}{100km}$. Take $S=((x_1,y_1),\ldots,(x_m,y_m))$ to be an increasing balanced sample such that the interval $I=\{x:x_{i-1}<x<x_i\}$ is of size $|\cX|-m$. Now, for every $x\in I$ denote by $S_x$ the sample obtained by replacing $x_{i}$ with $x$ in $S$ (with the same labeling). Notice that since $\cA$ is $(\frac{1}{100km})$-comparison-based, for every $x,x'\in I$
	\[\cA_{S_x,x'}(0)=\Pr_{h\sim\cA(S_x)}[0\notin h(x')]=\begin{cases}
		\leq \vec{p}^{(i-1)}_0+\frac{1}{100km} & x'<x \\
		\geq \vec{p}^{(i)}_0-\frac{1}{100km} & x'>x.
		\end{cases}\]

		Set $c=\frac{1}{2}\cdot(\vec p^{(i)}_0+\vec p^{(i-1)}_0)$. Therefore, 
		\begin{equation}\label{eq:bulding_dist_for_packing}
			\cA_{S_x,x'}(0)=\Pr_{h\sim\cA(S_x)}[0\notin h(x')]=\begin{cases}
			\leq c-\frac{1}{200km} & x'<x \\
			\geq c+\frac{1}{200km} & x'>x.
			\end{cases}
		\end{equation}
		    
		For every $ x \in I $, we define $ \mathcal{P}_x $ as the distribution $ \mathcal{A}(S_x) $, and we define $ E_x $ as the event 
		\[\left\{ h \in {\cY \choose k}^\cX : 0 \notin h(x) \right\}.\] 
		Note that for every $ x, x' \in I $, the datasets $ S_x $ and $ S_{x'} $ are neighboring. Since $ \cA $ is $ (\epsilon, \delta(m)) $-differentially private, the distributions $ \cP_x $ and $ \cP_{x'} $ are $ (\epsilon, \delta(m)) $-indistinguishable. Thus, the proof is complete.
		
		\end{proof} 
		
		We turn to prove \Cref{proposition:jump}.

		\begin{proof}[Proof of \Cref{proposition:jump}]
			The proof uses the assumption that $\cA$ empirically $k$-list learns the class of monotone functions with $k+1$ labels over $\cX$, $\cM_k(\cX)$. Let $S=\big((x_1,y_1),\ldots,(x_{m},y_{m})\big)$ be an increasing balanced realizable sequence. Recall, $\cA_{S,x}(j)=\Pr_{h\sim\cA(S)}[j\notin h(x)]$ for $j=0,1,\ldots,k$, and $\sum_{j=0}^k \cA_{S,x}(j) =1$. 
			The expected empirical loss of $\cA$ on $S$ is at most $\alpha+\beta $, since $\cA$ is $(\alpha,\beta)$-accurate empirical learner. Therefore,
			\begin{align*}
				1-(\alpha+\beta) & \leq \EEE{h\sim \cA(S)}{1- \loss{S}{h}}                                                                                           \\
				                 & =\sum_{j=0}^k \frac{1}{m}\left(\sum_{i=\frac{m}{k+1}\cdot j +1}^{\frac{m}{k+1}\cdot (j+1)} \Pr_{h\sim \cA(S)}[j\in h(x_i)]\right) \\  
				                 & =\sum_{j=0}^k \frac{1}{m}\left(\sum_{i=\frac{m}{k+1}\cdot j +1}^{\frac{m}{k+1}\cdot (j+1)} (1-\cA_{S,x_i}(j))\right)              
			\end{align*}
			By bounding each time $k$ different summands from above by $\frac{k}{k+1}$ we have for every $j=0,\ldots ,k$
			\[\frac{1}{m}\left(\sum_{i=\frac{m}{k+1}\cdot j +1}^{\frac{m}{k+1}\cdot (j+1)} (1-\cA_{S,x_i}(j))\right)\geq \frac{1}{k+1}-(\alpha+\beta).\]
			        
			Therefore, by a simple averaging argument, there exist $i_j\in\left[\frac{mj}{k+1} +1, \frac{m(j+1)}{k+1} \right]$ for $j=0,\ldots,k$ such that 
			
			\begin{equation}\label{eq:jump_eq1}
				1-\cA_{S,x_{i_j}}(j)\geq 1-(k+1)\cdot(\alpha+\beta).
			\end{equation}

			From now, on we will focus on $j=k$.
			Notice that 
			\[1-\cA_{S,x_{i_k}}(k)=\sum_{j'\neq k}\cA_{S,x_{i_k}}(j'),\]
			and therefore there exist $j'\neq k$ such that
			\[\cA_{S,x_{i_k}}(j')\geq \frac{1}{k}-\frac{k+1}{k}\cdot(\alpha+\beta).\]
			On the other hand, $\cA_{S,x_{i_{j'}}}(j')\leq (k+1)\cdot (\alpha+\beta)$ (by \Cref{eq:jump_eq1}), and therefore 
			\[|\cA_{S,x_{i_k}}(j')-\cA_{S,x_{i_{j'}}}(j')|\geq\frac{1}{k}-\frac{k+1}{k}\cdot(\alpha+\beta)-(k+1)\cdot (\alpha+\beta)=\frac{1}{k}-\frac{(k+1)^2}{k}(\alpha+\beta).\]
			
			Next, consider $S'$ the sample where we replace $x_{i_{j'}}$ with $x'$ satisfying $x_{i_{j'}-1}<x'<x_{i_{j'}}$, and $S''$ the sample where we replace $x_{i_{k}}$ with $x''$ satisfying $x_{i_{k}-1}<x''<x_{i_{k}}$. Note that $\mathtt{loc}_{S'}(x_{i_{j'}})=i_{j'}$ and $\mathtt{loc}_{S''}(x_{i_{k}})=i_{k}$.
			By privacy,
			\begin{align*}
				\cA_{S',x_{i_{j'}}}(j') & \leq e^\epsilon\cdot \cA_{S,x_{i_{j'}}}(j')+\delta\leq e^\epsilon\cdot(k+1)(\alpha+\beta)+\delta<\frac{1}{4k}, \tag{Holds for $\epsilon=0.1$, $\alpha=\beta=\frac{1}{200k(k+1)}$, $\delta<\frac{1}{100k}$}                                        \\
				\cA_{S'',x_{i_{k}}}(j') & \geq (\cA_{S,x_{i_{k}}}(j')-\delta)\cdot e^{-\epsilon}\geq \left(\frac{1}{k}-\frac{k+1}{k}\cdot(\alpha+\beta)-\delta\right)e^{-\epsilon}>\frac{1}{2k}.\tag{Holds for $\epsilon=0.1$, $\alpha=\beta=\frac{1}{200k(k+1)}$, $\delta<\frac{1}{100k}$} \\
			\end{align*}
			
			Now, since $\cA$ is $\left(\frac{1}{100km}\right)$-comparison-based we have
			\begin{align*}
				\vec{p}^{(i_{j})}_{j'} & <\frac{1}{4k}+\frac{1}{100km}, \\
				\vec{p}^{(i_k)}_{j'}   & >\frac{1}{2k}-\frac{1}{100km}. 
			\end{align*}
			Therefore, there exists $i_{j'}\leq i\leq i_k$ such that 
			\[|\vec{p}^{(i-1)}_{j'}-\vec{p}^{(i)}_{j'}|> \frac{1/4k}{m}-\frac{1}{50km^2}\geq \frac{3}{100km},\]
			which completes the proof.  
			        
		\end{proof}

		\subsection{Proof of \Cref{proposition:packing}}\label{sec:add_proof_prop:packing}
		
		\begin{proposition*}[Restatement of \Cref{proposition:packing}]
			Let $c\in[0,1], \gamma\leq\frac{1}{2}$, and
			let $\cP_1,\ldots, \cP_n$ be probability measures such that for all $i,j$, $\cP_i$ and $\cP_j$ are $(\epsilon,\delta)$-indistinguishable, for $\epsilon\leq0.1$ and 
			$\delta\leq\frac{1}{6\gamma^{-4}\log^2\left(\frac{1}{\gamma}\right)}$.
			Assume there exist events $E_1,\ldots, E_n$ such that for every $i,j$
			\begin{equation}\label{eq:threshold_prop}
				\cP_i(E_j)=\begin{cases}
				\leq c-\gamma & j<i \\
				\geq c+\gamma & j>i.
				\end{cases}
			\end{equation}
			Then, $n\leq 2^{\gamma^{-2}\log^2\left(\gamma^{-1}\right)}$.
		\end{proposition*}

		\begin{proof}
			  
			Set $T=\frac{1}{\gamma^2}\log^2\left(\frac{1}{\gamma}\right)-1$, $D=\frac{1}{\gamma^2}\ln T$.
			Assume towards contradiction that $n>2^{T+1}$. 
			We consider the following binary search\footnote{A binary search over a domain $A$ is modeled as a pair $(\cT,\{p_v\})$, where $\cT$ is a binary tree and $p_v:A\to{\pm 1}$ are predicates (assigned to internal vertices of $\cT$). The result of a search over an item $a\in A$ is the leaf defined by the root-to-leaf walk on the tree according to the answers $\{p_v(a)\}$.} 
			over probability measures (over the same $\sigma$-algebra as the $\cP_j$'s), defined using the events $E_j$'s. 
			In the first step of preforming the search on $P$, we consider $P(E_{n/2})$. If $P(E_{ n/2 })\leq c$, then we continue recursively with $E_{ n/2 +1},\ldots,E_n$. Otherwise, we continue recursively with $E_1,\ldots E_{ n/2 }$. 
			We perform this search for $T$ steps. 
			After $T$ steps, the result of the search is defined to be the index $j$ of the event $E_j$ that is supposed to be queried on the $T+1$ step. For example, for $T=1$, if $P(E_{n/2})\leq c$, then the result of the search is $3n/4$, and if $P(E_{n/2})> c$, the result is $n/4$.
			Observe that \begin{enumerate}
			\item [(i)] by the assumption, there are exactly $2^T$ different search outcomes, and
			\item[(ii)] for every possible outcome $i$, the result of applying the search over $\cP_i$ is exactly $i$. 
			\end{enumerate}

			Let  $\bar X=(X_1,\ldots,X_D)\sim \cP_i^D$ be $D$ IID samples from $\cP_i$.
			Denote the empirical measure\footnote{Given $D$ independent samples $X_1,\ldots, X_D$ from a distribution $P$, the \emph{empirical measure} is defined as $P_D(E)=\sum_{j=1}^D \delta_{X_j}(E)$, where $\delta_X$ is the Dirac measure.} 
			induced by $\bar X$ by $\cP_{i,\bar X}$.
			Next, consider performing the binary search over the empirical measure $\cP_{i,\bar X}$, where $i$ is a possible search outcome. We claim that with high probability, the result of the search will be exactly $i$. Quantitatively, denote
			by $F_i$ the event
			\[F_i=\{ \bar{X} : \text{performing binary search on the empirical measure defined by $\bar X$ yields $i$}\}.\] 
			By a standard application of 
			Chernoff and union bound,
			\begin{align*}\label{eq:chernoff}
				\cP_i^D(F_{i}) & \geq 1-T\cdot\Pr_{\bar X\sim \cP_i^D}[|\cP_{i,\bar X}(E_k)-\cP_i(E_k)|\geq \gamma]\notag \\
				               & \geq 1-T\exp{(-2\gamma^{2}D)}                                                            \\
				               & =1-\frac{1}{T}>\frac{2}{3},                                                              
			\end{align*}
			where $k$ is some index along the query branch.
			
			On the other hand, since $\cP_i^D$ and $\cP_j^D$ are $(D\epsilon, D\delta)$-indistinguishable, we have
			\begin{align*}
				\cP_j^D(F_{i}) & \geq \exp{(-D\epsilon)}\cdot(\cP_i^D(F_i)-D\delta)           \\
				               & \geq \exp{(-D\epsilon)}\cdot\left(\frac{2}{3}-D\delta\right) \\
				               & \geq \frac{1}{2}\exp{(-D\epsilon)},                          
			\end{align*}
			where the last inequality holds by the choice of $\delta$. 
			Recall, there are $2^T$ different search outcomes (Observation (i)), and therefore the events $F_i$'s are mutually disjoint. Therefore,
			\begin{align*}
				1 & \geq  \cP_j^D(\cup_i F_i)                                             \\
				  & =\sum \cP_j^D( F_i)                                                   \\
				  & \geq 2^T\cdot\frac{1}{2}\exp{(-D\epsilon)}=2^{T-1}\exp{(-D\epsilon)}, 
			\end{align*}
			
			however, this is a contradiction due to the choice of $T,D,\epsilon$. 

		\end{proof}

		\section{Ramsey Theorem for $b$-ary Trees}\label{sec:ramsey_theorem__for_trees}
		In this section, we present the proof of \Cref{thm:ramsey_trees}, as given in \cite{FioravantiHMST24ramsey}, with the necessary modifications to extend it to general $b$-ary trees.
		
		Denote by $\Ramsey_m(d;k,b)$ the smallest $n$ that satisfies the condition in \Cref{thm:ramsey_trees}, i.e.\ $\Ramsey_m(d;k,b)$ is the smallest $n$ such that for every coloring of $m$-chains with $k$ colors, there exists a type-monochromatic subtree of depth $d$. 
		The proof of \Cref{thm:ramsey_trees} consists of two parts: the first part shows that $\Ramsey_m(d;k,b)$ is well defined (i.e.\ $\Ramsey_m(d;k,b)<\infty$), while the second part proves the quantitative upper bound stated in theorem. 
		We start by citing the pigeonhole principle for trees, which serves as the base case for the inductive proof for coloring $m$-chains for any value of $m$.
		
		\begin{proposition}[Pigeonhole Principle for Trees, \cite{AlonBLMM22}]\label{prop:php}
			Let $d\in\bbN$ and let $T$ be a complete $b$-ary tree of depth $n$. Then, for every coloring of its vertices with $k$ colors, the following hold:
			If $T$ has depth $n \ge dk$, it admits a subtree $S$ of depth $d$;\label{item:php}
		\end{proposition}
		
		\begin{proof}[Proof of \Cref{thm:ramsey_trees}]
			We prove the theorem by induction on $m$.
			The case $m=1$ follows immediately from \Cref{prop:php} since a $1$-chain is simply a vertex, hence a $1$-chain coloring is a vertex coloring, and the promised monochromatic subtree is in particular type-monochromatic.
			\medskip
				
			Assume that the statement holds for $m-1$, and denote ${t=\Ramsey_{m-1}(d;k^b,b)}$.  
			Let $T$ be a complete $b$-ary tree of depth $n$, where $n$ is sufficiently large (the size of $n$ will be determined later on in \Cref{app:upper_bound_ramsey_chains}).
			Let $\chi$ be an $m$-chain coloring of $T$ using $k$ colors.
			We introduce a recursive procedure constructing a subtree of $T$ of depth $t$, denoted $T^\star$.  
			Then, we define an $(m-1)$-chain coloring $\chi^\star$ of $T^\star$, such that any $\chi^\star$-type-monochromatic subtree of $T^\star$ is in fact type-monochromatic with respect to~$\chi$.
			Finally, we apply the induction hypothesis on $T^\star$ and $\chi^\star$, allowing us to obtain the desired type-monochromatic subtree.
			Throughout the proof we denote the vertices of $T^\star$ by $u_\sigma$, where for a string $\sigma\in[b]^i$, $i\in\{0,\ldots, t-1\}$, and $r\in [b]$, $u_\emptyset$ is the root of $T^\star$ and $u_{\sigma r}$ is the $r$'th child of $u_\sigma$.
			
			\medskip
			    
			We define by induction a sequence of trees $S_\sigma$ and vertices $u_\sigma$ as follows.
			For every $\sigma\in [b]^{\leq(m-2)}$, set $u_\sigma$ to be the vertex in $T$ represented by the sequence~$\sigma$ (the root of $T$ is represented by the empty sequence).
			Next, for every $\sigma\in [b]^{m-2}$ set $S_\sigma$ to be the subtree of $T$ rooted at $u_\sigma$. 
			Assume the subtree $S_\sigma$ has been defined and $u_\sigma$ is the root of $S_\sigma$, where $\sigma\in[b]^i$ is a $b$-ary sequence of length~$i\geq m-2$.
			We define subtrees $S_{\sigma r}$ and vertices $u_{\sigma r}$ where $r\in[b]$, as follows.
			
			\begin{enumerate}
				\item Consider the $r$'th-subtree of $S_\sigma$, that is the subtree of $S_\sigma$ emanating from the $r$'th child of the root $u_\sigma$.
				      Define an equivalence relation on the vertices of the $r$'th-subtree of $S_\sigma$ as follows:
				      \begin{equation*}
				      	\begin{gathered}
				      		x\equiv y \\
				      		\iff \\
				      		\forall A\subset \{u_{\sigma(0)}, u_{\sigma(1)},\ldots, u_{\sigma(i-1)}\}, |A|=m-2 ~:~ \chi(A\cup\{u_\sigma,x\})=\chi(A\cup\{u_\sigma,y\}),
				      	\end{gathered}
				      \end{equation*}
				      where $\sigma(j)$ is the prefix of $\sigma$ of length $j$ (with $\sigma(0)$ being the empty sequence).
				      Note that indeed every choice of $m-2$ vertices from the set $\{u_{\sigma(0)}, u_{\sigma(1)},\ldots, u_{\sigma(i-1)}\}$, together with $u_\sigma$ and  a vertex from the $r$'th-subtree of $S_\sigma$, form an $m$-chain in $T$.
				      
				      Observe that an equivalence class is determined by a sequence of ${i} \choose {m-2}$ colors, therefore there are at most $k^{ {i} \choose {m-2}}$ such equivalence classes. 
				      
				\item  Apply the pigeonhole principle for trees (\Cref{prop:php}) on the $r$'th-subtree of $S_\sigma$, where the colors of the vertices are the equivalence classes defined in the previous step. Set~$S_{\sigma r}$ to be the promised monochromatic subtree, and set~$u_{\sigma r}$ to be its root.
			\end{enumerate}
			
			We choose $n$ to be sufficiently large so this procedure may be continued until  $T^\star=\{u_\sigma\}_{\sigma\in[b]^{\leq t}}$ have been defined. See \Cref{app:upper_bound_ramsey_chains} for a more detailed discussion.
			Note that for every $b$-ary sequence $\sigma$ of length $i$, and every $r\in[b]$,
			
			\begin{equation}\label{eq:recursive_eq_inside_proof}
				\mathtt{depth}(S_{\sigma r})\geq \left\lfloor\frac{\mathtt{depth}(S_\sigma) -1}{k^{ {i} \choose {m-2}}}\right\rfloor.
			\end{equation}
			
			Next, define an $(m-1)$-chain coloring of $T^\star$, denoted $\chi^\star$, as follows.
			\[\forall \text{ $(m-1)$-chain $C$ in $T^\star$}~:~ \chi^\star(C)=\bigl(\chi(C\cup \{u_{\sigma_0}\}),\chi(C\cup \{u_{\sigma_1}\}),\ldots,\chi(C\cup \{u_{\sigma_{b-1}}\})\bigr),\]
			where $u_{\sigma_r}$ is any vertex from $T^\star$ that belongs to the $r$'th subtree emanating from the last vertex of the chain $C$. (If the last vertex in $C$ is at level $t$, meaning it is a leaf of $T^\star$, we just pick an arbitrary color for $C$ out of the $k^b$ possible colors.)
			Note that $\chi^\star$ is well-defined. Indeed, take an $(m-1)$-chain $C$ in $T^\star$. If $x,y$ are vertices that belong to the $r$'th subtree emanating from the last vertex of the chain $C$, then by construction $x\equiv y$, meaning
			$\chi(C\cup\{x\}) = \chi(C\cup\{y\})$.
			
			\medskip
			
			Finally, by the induction hypothesis applied on $T^\star$ and $\chi^\star$, and by the choice of $t={\Ramsey_{m-1}(d;k^b,b)}$, there exists a type-monochromatic subtree of $T^\star$ of depth~$d$. 
			In fact, this subtree is also type-monochromatic with respect to the original $m$-chain coloring $\chi$. 
			Indeed, if $C$ and $C'$ are two $m$-chains with the same type, then the first $m-1$ vertices of $C$ and $C'$ form an $(m-1)$-chains with the same type that have the same color with respect to $\chi^\star$. So by the definition of $\chi^\star$ it is affirmed that $\chi(C)=\chi(C')$.  
			Therefore, we proved the finiteness of the Ramsey number $\Ramsey_m(d;k,b)$. 
			It is left to obtain from the recursive procedure described here the upper bound for $\Ramsey_m(d;k,b)$ that is stated in the theorem. 
			We provide a detailed calculation in \Cref{sec:upper_bound_ramsey_chains}.
		\end{proof}
		
		\subsection{Quantitative Bound for Ramsey Number}\label{sec:upper_bound_ramsey_chains}
		Here we give an explicit calculation for the sufficient depth~$n$ that is required for the proof. Subsequently, we derive an upper bound for~$\Ramsey_m(d;k,b)$.
		
		The procedure described in the proof of \Cref{thm:ramsey_trees} can be continued until step ${t=\Ramsey_{m-1}(d;k^b,b)}$ if for every sequence~$\sigma$ of length~$t$, $\mathtt{depth}(S_{\sigma}) \ge 0$. Consider a sequence~$\sigma$ of length~$t$. To ease the notation, for every step $i\in\{m-2,\ldots,t\}$, $\mathtt{depth}(S_{\sigma(i)})$ is denoted by~$d_i$.
		Recall, by \Cref{eq:recursive_eq_inside_proof} the following holds.
		\begin{equation}\label{eq:inductive_theOG}
			\left\{
			\begin{array}{ll}
				d_{m-2} & = n-(m-2),                                                       \\
				d_{i+1} & \ge \left\lfloor\frac{d_i -1}{k^{i \choose {m-2}}}\right\rfloor. 
			\end{array}
			\right.
		\end{equation}
		Observe that if $d_i \ge 2 k^{i \choose {m-2}}$ then
		\begin{equation}\label{eq:recurrence_relation}
			d_{i+1}\geq 
			\left\lfloor\frac{d_i -1}{k^{i \choose {m-2}}}\right\rfloor \ge \frac{d_i}{k^{i \choose {m-2}}} - 1 \ge \frac{d_i}{2k^{i \choose {m-2}}}.
		\end{equation}
		If $d_i < 2 k^{i \choose {m-2}}$, then
		\[
			\left\lfloor\frac{d_i -1}{k^{i \choose {m-2}}}\right\rfloor \in\{0,1\} \;,
		\]
		meaning that the procedure terminates, or continues for one more last step, therefore we can assume that the bound in \Cref{eq:recurrence_relation} holds in every step $i$.
		By induction, using the recurrence relation in \Cref{eq:inductive_theOG,eq:recurrence_relation},
		\begin{equation}\label{eq:absolute_bound_on_depth}
			d_i \ge \frac{n-(m-2)}{2^{i-(m-2)}k^{\sum_{j=m-2}^{i-1} \binom{j}{m-2} }} = \frac{n-(m-2)}{2^{i-(m-2)}k^{ \binom{i}{m-1} }} \;,
		\end{equation}
		because $\sum_{j=m-2}^{i-1} \binom{j}{m-2} = \binom{i}{m-1} $ (the left-hand-side counts the number $(m-1)$-subsets $S\subseteq [i]$, by partitioning them according to the largest element).
		For the procedure to continue $t$ steps we require $d_t\geq1$.
		Together with \Cref{eq:absolute_bound_on_depth} we deduce the following bound:
		\begin{equation*}
			n \ge 2^{t-(m-2)}k^{ \binom{t}{m-1} } + (m-2).
		\end{equation*}
		Notice that for every $m,k\geq2$ and $t\geq m-2$,
		\begin{equation*}
			2^{t-(m-2)}k^{ \binom{t}{m-1} } + (m-2) \leq 2^{t-(m-2)}k^{t^{m-1}} + (m-2)\leq k^{2t^{m-1}}.
		\end{equation*}
		
		Therefore, choosing $n=k^{2t^{m-1}}$ is sufficient.
		Recall that $t=\Ramsey_{m-1}(d;k^b,b)$, therefore the following recursive relation is obtained.
		\begin{equation}\label{eq:ramsey_recursive_bound}
			\Ramsey_m(d;k,b)\leq k^{2\Ramsey_{m-1}(d;k^b,b)^{m-1}}.
		\end{equation}
		
		From now on we will use the Knuth notation $a \uparrow b$ in place of $a^b$ to ease the calculations, and recall that the Knuth's operator is right-associative, i.e.\ $a \uparrow b\uparrow c = a \uparrow (b\uparrow c)$.
		By applying \Cref{eq:ramsey_recursive_bound} repeatedly we obtain the following bound
		
		\begin{align*}
			  & \Ramsey_m(d;k,b)                                                                                                                                                                                \\
			  & \le k^2 \uparrow \Ramsey_{m-1}(d;k^2,b) \uparrow (m-1)                                                                                                                                          \\
			  & \leq  k^2 \uparrow (k^{2\cdot b} \uparrow (m-1)) \uparrow \Ramsey_{m-2}(d;k^{b^2},b) \uparrow (m-2)                                                                                             \\
			  & \leq  k^2 \uparrow (k^{2\cdot b} \uparrow (m-1)) \uparrow (k^{2\cdot b^2} \uparrow (m-2)) \uparrow \Ramsey_{m-3}(d;k^{b^3},b) \uparrow (m-3)                                                    \\
			  & \leq \ldots                                                                                                                                                                                     \\
			  & \leq k^2 \uparrow (k^{2\cdot b} \uparrow (m-1)) \uparrow (k^{2\cdot b^2} \uparrow (m-2)) \uparrow \ldots \uparrow (k^{2\cdot b^{m-2}} \uparrow 2) \uparrow \Ramsey_1(d;k^{b^{m-1}},b)\uparrow 1 \\
			  & = k^2 \uparrow (k^{2\cdot b} \uparrow (m-1)) \uparrow (k^{2\cdot b^2} \uparrow (m-2)) \uparrow \ldots \uparrow (k^{2\cdot b^{m-2}} \uparrow 2) \uparrow dk^{b^{m-1}},                           
		\end{align*}
		where $\Ramsey_1(d;k,b)=dk$ by the pigeonhole principle (\Cref{prop:php}).
		Denote 
		\begin{equation*}
			R_i= \begin{cases}
			dk^{b^{m-1}}\coloneqq R  & \text{if $i=1$;}\\[10pt]
			k^{2\cdot b^{m-i}\cdot i \cdot R_{i-1}}
			& \text{if $1<i<m$;}\\[10pt]
			k^{2\cdot R_{m-1}}
			& \text{if $i=m$.}
			\end{cases}
		\end{equation*}
		Using this notation, $\Ramsey_m(d;k,b) \leq R_m$.
		
		\begin{proposition}\label{proposition:bounding_with_twr}
			For $2\leq i\leq m$,
			\[R_i\leq \twr_{(i)}(c_i \cdot b^{m-2}R\log k),\]
			where\footnote{We use the convention $\sum_{k=n_1}^{n_2}f(k)=0$ if $n_2<n_1$, hence $c_2=4$.} 
			\[c_i=4+\sum_{j=3}^i{\frac{\max\{1,\log_{(j-2)}(2\cdot b^{m-j}j\log k)\}}{b^{m-2}R\log k}}.\]
		\end{proposition}
		
		\begin{proof}[Proof of \Cref{proposition:bounding_with_twr}]
			Proof by induction on $i$.
			For $i=2$, 
			\[R_2=k^{4\cdot b^{m-2} \cdot R}=\twr_{(2)}(c_2\cdot b^{m-2}R\log k).\]
			For $2<i<m$,
			\begin{align*}
				R_i 
				  & = k^{2\cdot b^{m-i}\cdot i \cdot R_{i-1}}                                                                                                                      \\
				  & =\twr_{(2)}\left[2\cdot b^{m-i} i\log k \cdot R_{i-1}\right]                                                                                                   
				\\
				  & \leq \twr_{(2)}\left[2\cdot b^{m-i}\cdot i\log k\cdot\twr_{(i-1)}(c_{i-1}\cdot b^{m-2}R\log k)\right] \tag{by induction.}                                      
				\\
				  & =\twr_{(2)}\left[\twr_{(i-1)}\log _{(i-2)}(2\cdot b^{m-i} i\log k)\cdot \twr_{(i-1)}(c_{i-1}\cdot b^{m-2}R\log k)\right]                                       
				\\
				  & \leq\twr_{(2)}\left[\twr_{(i-1)}\left(\max \{1,\log _{(i-2)}(2\cdot b^{m-i} i\log k)\}\right)\cdot \twr_{(i-1)}\left(c_{i-1}\cdot b^{m-2}R\log k\right)\right] 
				\\
				  & \leq \twr_{(2)}\left[\twr_{(i-1)}(\max \{1,\log _{(i-2)}(2\cdot b^{m-i} i\log k)\})+c_{i-1}\cdot b^{m-2}R\log k)\right]                                        
				\tag{$\star$}
				\\
				  & =\twr_{(2)}\left[\twr_{(i-1)}(c_i \cdot b^{m-2}R\log k)\right]                                                                                                 
				\tag{definition of $c_i$.}
				\\
				  & =\twr_{(i)}(c_i \cdot b^{m-2}R\log k),                                                                                                                         
			\end{align*}
			where the inequality marked with $(\star)$ holds since 
			\[\twr_{(n)}(x)\cdot\twr_{(n)}(y)\leq \twr_{(n)}(x\cdot y)\]
			for $x,y\geq1, n\geq2$.
			    
			The case $i=m$ follows because $R_m=k^{2\cdot R_{m-1}}\leq k^{2m\cdot R_{m-1}}$, and $k^{2m\cdot R_{m-1}}\leq \twr_{(m)}(c_m \cdot b^{m-2}R\log k)$, by using the above calculation for $2<i<m$ one more time for $i=m$.
			    
		\end{proof}
		
		\begin{corollary}[Upper Bound For Ramsey Number for Chains]\label{cor:upper_bound_for_Ramsey_number}
			For every integers $m\geq 2,d\geq m,k\geq 2$,$b\geq 2$,
			\begin{equation*}
				\Ramsey_m(d;k,b)\leq \twr_{(m)}(5\cdot b^{m-2}dk^{b^{m-1}}\log k).
			\end{equation*}
		\end{corollary}
		
		\begin{proof}
			It suffices to bound
			\[c_m= 4+\sum_{j=3}^m{\frac{\max\{1,\log_{(j-2)}(2\cdot b^{m-j}j\log k)\}}{b^{m-2}R\log k}}.\]  
			Note that for every $3\leq j\leq m$, 
			\[\log_{(j-2)}(2\cdot b^{m-j}j\log k)\leq 2\cdot b^{m-j}\log k\leq 2\cdot b^{m-2}\log k.\]
			Therefore,
			\[c_m\leq 4+ \sum_{j=3}^m{\frac{2\cdot b^{m-2}\log k}{b^{m-2}R\log k}}= 4+\frac{2(m-2)}{dk^{b^{m-1}}}\leq 5.\]
		\end{proof}

\end{document}